%% file: main.tex
\documentclass{article}
\usepackage{microtype}
\usepackage{graphicx}
\usepackage{booktabs} 
\usepackage{hyperref}

\usepackage{xr} 
\usepackage{amsmath,amsthm,amssymb,amsfonts}
\usepackage[inline]{enumitem}
\usepackage{xfrac}
\usepackage{todonotes}
\usepackage{subcaption}
\usepackage{caption}
\usepackage{dsfont}
\usepackage{makecell}
\input{abbrevs}

\usepackage[accepted]{icml2020}

\icmltitlerunning{Estimation of Bounds on Potential Outcomes For Decision Making}

\begin{document}

\twocolumn[
\icmltitle{Estimation of Bounds on Potential Outcomes For Decision Making}

\icmlsetsymbol{equal}{*}

\begin{icmlauthorlist}
\icmlauthor{Maggie Makar}{csail}
\icmlauthor{Fredrik Johansson}{fj}
\icmlauthor{John Guttag}{csail}
\icmlauthor{David Sontag}{csail}
\end{icmlauthorlist}

\icmlaffiliation{csail}{CSAIL, MIT}
\icmlaffiliation{fj}{Chalmers University of Technology}

\icmlcorrespondingauthor{Maggie Makar}{mmakar@mit.edu}

\icmlkeywords{machine learning, causal inference, causality}

\vskip 0.3in
]

\printAffiliationsAndNotice{}  

\begin{abstract}
Estimation of individual treatment effects is commonly used as the basis for contextual decision making in fields such as healthcare, education, and economics. However, it is often sufficient for the decision maker to have estimates of upper and lower bounds on the potential outcomes of decision alternatives to assess risks and benefits. We show that, in such cases, we can improve sample efficiency by estimating simple functions that bound these outcomes instead of estimating their conditional expectations, which may be complex and hard to estimate. Our analysis highlights a trade-off between the complexity of the learning task and the confidence with which the learned bounds hold. Guided by these findings, we develop an algorithm for learning upper and lower bounds on potential outcomes which optimize an objective function defined by the decision maker, subject to the probability that bounds are violated being small. Using a clinical dataset and a well-known causality benchmark, we demonstrate that our algorithm outperforms baselines, providing tighter, more reliable bounds. 
\end{abstract}

\input{contents/introduction.tex}

\input{contents/related_work.tex}

\input{contents/background.tex}

\input{contents/generalization.tex}

\input{contents/alg.tex}

\input{contents/experiments.tex}

\input{contents/conclusion.tex}

\bibliography{bpout}
\bibliographystyle{icml2020}

\input{contents/supplement}

\end{document}

%% file: abbrevs.tex
\newtheorem{thmdef}{Definition}

\newtheorem{thmthm}{Theorem}
\newtheorem{thmasmp}{Assumption}

\newtheorem{thmappdef}{Definition}

\newtheorem{thmapplem}{Lemma}

\newtheorem{thmappcol}{Corollary}

\newtheorem{thmappthm}{Theorem}

\newenvironment{thmproof}[1][Proof]{\begin{trivlist}
\item[\hskip \labelsep {\textit{#1.}}]}{\end{trivlist}}
\newenvironment{thmproofsketch}[1][Proof sketch]{\begin{trivlist}
\item[\hskip \labelsep {\textit{#1.}}]}{\end{trivlist}}

\def\E{\mathbb{E}}
\def\indic{\mathds{1}}
\def\cN{\mathcal{N}}

\def\cD{\mathcal D}
\def\cC{\mathcal C}

\def\cX{\mathcal X}
\def\cY{\mathcal Y}

\def\cF{\mathcal F}

\def \R{\mathbb{R}}

\def\bsf{\boldsymbol{f}}
\def\iw{\mathrm{IW}}

\newcommand\indep{\protect\mathpalette{\protect\independenT}{\perp}}
\def\independenT#1#2{\mathrel{\rlap{$#1#2$}\mkern2mu{#1#2}}}

\newcommand\bdt[1][.8]{\mathbin{\vcenter{\hbox{\scalebox{#1}{$\bullet$}}}}}

%% file: contents/introduction.tex
\section{Introduction}
In many practical situations, a decision maker wishes to intervene or assign a treatment to ensure that an outcome of interest falls within a safe range. One example, which we use throughout the paper, is when a physician considers whether or not to prescribe anticoagulants to mitigate the risk of stroke, as measured by the International Normalized Ratio (INR). The INR reflects the time it takes for blood to clot. For previous stroke patients, a healthy INR is 2--3. Values lower than 2 signal elevated risk of an Ischemic stroke, and higher than 3 signal elevated risk of a Hemorrhagic stroke. To make an informed decision, the physician needs to know if the potential outcomes under treatment and non-treatment fall within 2--3. Learning that the difference between the potential outcomes, i.e., the Individual Treatment Effect (ITE) is 1.5, does not immediately imply an optimal treatment decision; it could be that the patient's INR decreases from 4 to 2.5 or from 5.5 to 4. More information about the potential outcomes themselves is needed, but knowing their exact value is not necessary. It is sufficient to know that the patient's INR is somewhere between 2.1 and 2.9 if treated. For example, knowing that it is 2.853 does not provide additional insight. For these two reasons, we study the task of estimating reliable covariate-conditional bounds on potential outcomes using observational data. 

Most existing methods for estimating causal effects and potential outcomes attempt to fit the expected outcomes as functions of observed covariates, typically relying on variants of Empirical Risk Minimization (ERM) strategies \citep{hill2011bayesian, shalit2017estimating, alaa18, alaa17}. Some of these methods produce prediction intervals centered around the estimated expected response (outcome) surface, which can be used to bound the potential outcome from above and below. These intervals have approximately valid coverage for large samples, provided that the mean estimate is sufficiently unbiased. However, achieving this is not always feasible in small samples, leading to high false coverage rates (FCRs), defined as the rate at which outcomes are observed outside of the given prediction interval.

Instead of attempting to directly fit the potential outcomes, which may be complex and hard to estimate from small samples, we propose to fit simpler functions that bound the outcomes from above and below. Within this simpler function class, we identify estimates of the potential outcomes that maximize a utility (objective) function specified by the decision maker. Figure~\ref{fig:illust} shows the intuition behind our approach. For example, if the decision maker wants to ensure that the uncertainty in the potential outcome estimates is small on average, they could require that the average interval width (= upper bound - lower bound) is small. Alternatively, if they wish to ensure that no patient sub-population has excessively uncertain estimates (i.e., wide intervals) they could require that the maximum interval width is minimized. 

\begin{figure}[t!]
    \centering
    \includegraphics[width=.5\columnwidth]{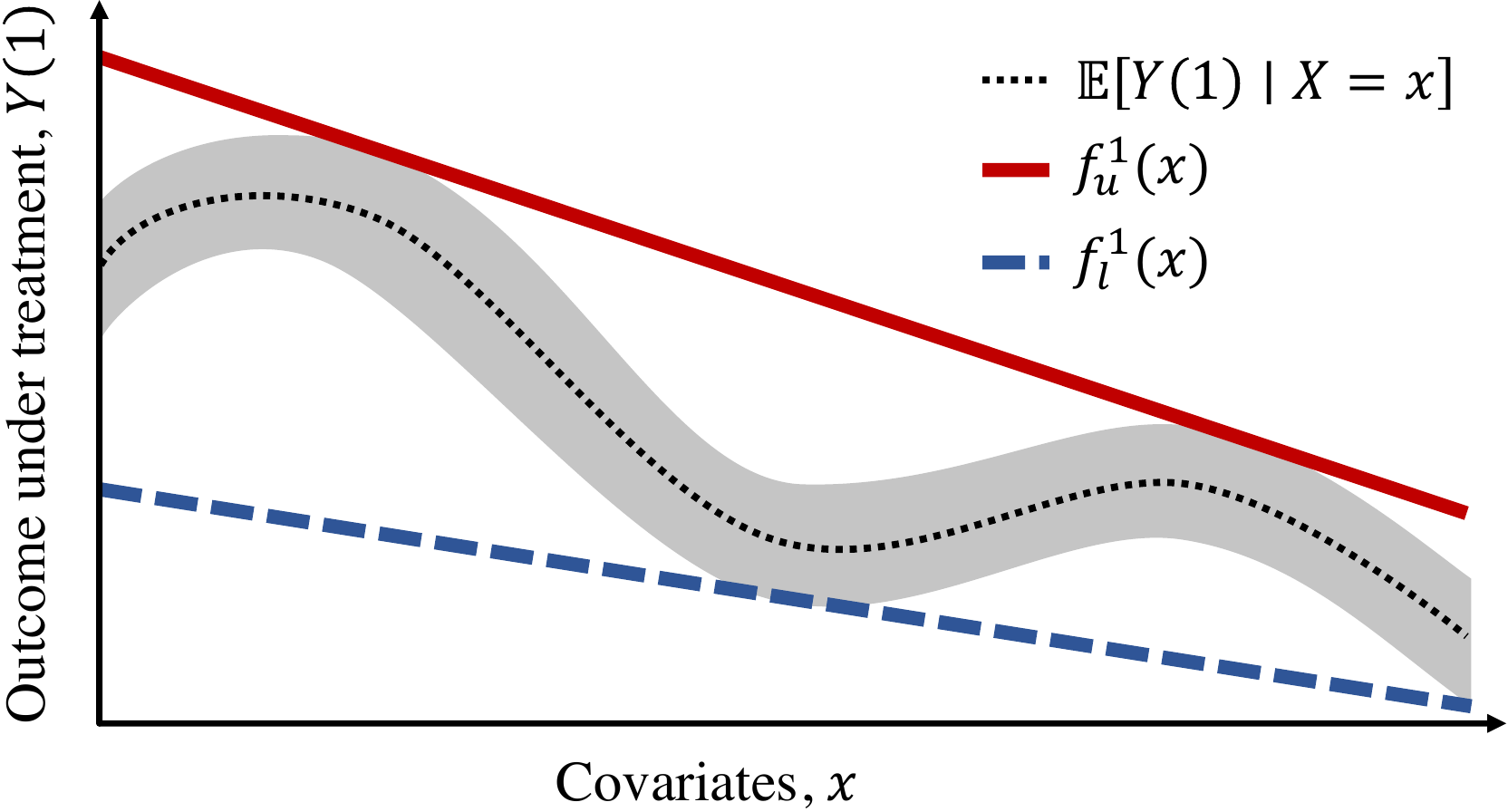}\vspace{-.5em}%
    \caption{Illustration of the intuition behind our theoretical findings. While the true potential outcome (black/gray) belongs to a complex class, the upper (red) and lower (blue) bounds $f_u^1, f_l^1$ that correctly cover it belong to a simple linear function space. \vspace{-.5em} \label{fig:illust}}
\end{figure}

We make the following main contributions:
\begin{enumerate*}[label=(\roman*)]
\item We give results on the generalization properties of learned bounds on potential outcomes and the conditions under which estimation of such bounds yields better sample complexity than fitting the expected outcomes using standard risk minimization methods. Our analysis highlights a trade-off between reliability (i.e., the probability that the bounds correctly cover the data) and the complexity of the learning task. 
\item We design an algorithm that finds the optimal bound estimates that maximize a given utility or objective function while providing reliable bounds. We explore different objective functions, analyzing the differences between the resulting bounds, and prove equivalence to quantile regression in a special case.
\item We evaluate our algorithm on a semi-synthetic clinical dataset and a well-known causality benchmark. We show how it can guide treatment decisions, and that it achieves a better trade-off between bound violations and utility than baseline algorithms. 
\end{enumerate*}

%% file: contents/related_work.tex
\section{Related work}
%
Research into methods for estimating conditional causal effects has focused primarily on estimating the expected potential outcomes or conditional average treatment effect (CATE) as functions of observed covariates~\citep{dorie2019automated}. For example, \citet{alaa18} showed that the CATE estimation problem is as hard as modelling the more complex of the two potential outcomes in the minimax sense. Similarly, \citet{nie2017quasi} show asymptotic bounds that rely on the complexity of the underlying function class of the CATE. More generally, recent work in CATE estimation has focused on the learning challenges associated with the difference between the treated and control populations, and on improving finite sample efficiency by sharing data between treatment groups~\citep{johansson2016learning,shalit2017estimating,alaa17,hill2011bayesian}. In contrast, we aim to improve sample efficiency by providing bounds on the causal estimands.

Other work focuses on estimating lower or upper bounds of Average Treatment Effect (ATE), to account for the possibility of unobserved confounding~\citep{balke1997bounds,bareinboim2012controlling,pearl2009causality,cai2008bounds}. Recently, this type of analysis was extended to include bounds on CATE, but again in the presence of hidden confounding \citep{kallus2018interval}. This line of work falls under sensitivity analysis \cite{rosenbaum2014sensitivity}, which is distinct from our work in that we aim to find bounds on the potential outcomes even in the absence of unobserved confounding. 


Another related line of work is the problem of conditional quantile treatment effect estimation \citep{koenker1978regression,chernozhukov2005iv}. Like our method, quantile methods give can give approximate bounds on the potential outcomes. The distinction is that the main objective of our method is not to estimate the specific quantile of treatment effect, but rather to provide the simplest functions that bound the outcomes such that an objective function given by the decision maker is optimized; we do not wish in general to establish asymptotic convergence to a particular quantile of the treatment effect. However, as we prove later, quantile estimation is a special case of our setting for a certain objective function. 

At the time of publication of this paper, new work extended conformal intervals \cite{lei2018distribution} to settings similar to ours, where the outcomes are counterfactual \cite{lei2020conformal}. Our work is distinct from the work presented in \cite{lei2020conformal} in three ways (1) we provide theoretical guarantees for the \emph{finite} sample rather than asymptotic regime, (2) our theoretical analysis highlights a fundamental trade-off between the statistical complexity of the learning problem and the confidence with which the learned interval truly covers the potential outcomes. Finally, (3) our approach allows for a more general definition of interval optimality; we not assume that tightness of the bounds is the only important metric to be optimized, but it allows the decision maker to define their own desiderata for optimality (e.g., fairness).

Our work is related to offline policy learning (e.g., \citet{swaminathan15normalized,swaminathan2015counterfactual}). The main difference between this work and ours is that we wish to obtain bounds for the potential outcomes, not just an optimal policy. This allows the decision maker to consider the estimated effect of the treatment against a backdrop of additional information that may not be recorded in the observational data.  
\vspace{-.5em}

%% file: contents/background.tex
\section{Background}

We consider learning of bounds on potential outcomes from finite-sample observational data, adopting the notation of the Neyman-Rubin potential outcomes framework \citep{rubin2011causal}. For each unit $i$ (e.g. patient), we observe a set of features $X_i \in \cX$, with $\cX$ a bounded subset of $\R^d$, an action (also known as treatment or intervention) $T_i \in \{0,1\}$ and an outcome $Y_i \in \R$. We observe these variables through samples $(x_1, t_1, y_1), ..., (x_n, t_n, y_n) \stackrel{i.i.d.}{\sim} p(X, T, Y)$ and denote by $n_t = \sum_{i=1}^n \indic\{t_i = t\}$ the number of observed samples for  treatment group $t\in \{0,1\}$, and let $p_t(X) = p(X\mid T=t)$. The observed outcome is one of the two \emph{potential outcomes}, $Y(0)$ and $Y(1)$, under control ($T=0$) and treatment ($T=1$), respectively. We use $\| a\|_p$ to denote the $p-$norm of a vector a. When the subscript is omitted, we refer to the 2-norm. 

We seek to learn high-probability bounds on both potential outcomes, $Y(0)$ and $Y(1)$, conditioned on the set of observed features $X$. Since only one outcome is observed, the other is not identifiable without strong assumptions. To that end, we assume that the features $X$ are sufficient to deconfound estimates of $Y(0), Y(1)$:
\begin{thmasmp}\label{assump:id}
The features $X$, treatment $T$ and potential outcomes $Y(0), Y(1)$ satisfy for some $\epsilon > 0$

\begin{enumerate}[topsep=0pt,itemsep=-0.5ex,partopsep=1ex,parsep=1ex]
    \item Strong ignorability: $Y(0), Y(1) \indep T \mid X$
    \item Overlap:  $\forall x, t : p(T=t \mid x) > \epsilon$
    \item Consistency: $Y = Y(T)$
\end{enumerate}

\end{thmasmp}
Under Assumption~\ref{assump:id}, $p(Y(t)=y \mid X=x) = p(Y=y \mid T=t, X=x)$~\citep{imbens2009recent}. This means that the distribution of potential outcomes can be estimated through regression or other standard methods. 
When treatment and outcomes are confounded, estimates of causal effects exhibit bias. For example, if medication A was prescribed more often to terminally ill patients than the alternative treatment B, we might learn that the life expectancy on treatment A was lower than on B, regardless of its average causal effect. To undo this bias, it is common to use the propensity score $e(x,t) := p(T=t\mid X=x)$ to re-weight the cohort using importance weighting.

\begin{thmdef}\label{def:pscore}
    The importance weighting function $w_t$ for group $t\in \{0,1\}$ is
    $
    w_t(x) := p(T=t) / e(x,t) ~.
    $
\end{thmdef}
We use $w_i$ to denote $w_{t_i}(x_i)$ for a sample $(x_i, t_i)\sim p$. With $w_t$ as in Definition~\ref{def:pscore}, we have for an arbitrary function $f$ on $\cX$ (e.g., the expected outcome or a prediction loss),
$
\E_X[f(X)] = \E_{X\mid T}[w_t(X) f(X) \mid T=t]~.
$
By Assumption~\ref{assump:id}, we have that the importance weights are bounded, meaning that for some $C_t < \infty$ and $t \in \{0, 1\}$:
\begin{align}\label{eqn:cortes_c}
    \sup_{x\in \cX} w_t(x) = \sup_{x\in \cX} \frac{p(T=t)}{e(x, t)} = 2^{D_\infty(p|| p_t)} = C_t, 
\end{align}
where $D_k(p || q)$ is the k$^{\text{th}}$-order R\'enyi divergence, and the second equality follows by applying the Bayes rule, and the definition of the R\'enyi divergence. It will be convenient to denote $2^{D_k(p||q)}$ by $d_k(p||q)$. 
Since $2^{D_{k-1}(p|| p_t)} < 2^{D_k(p|| p_t)}$, we have $d_2(p|| p_t) < C_t$.

%% file: contents/generalization.tex
\section{Generalization of bounds on potential outcomes}\label{sect:gen_ge}

Our goal is to estimate four functions; lower and upper bounds for the potential outcome under treatment, $\boldsymbol{f^1}(x) = \{f^1_l(x),f^1_u(x)\}$, and similarly defined functions for the outcome under control $\boldsymbol{f^0}(x) = \{ f^0_l(x),f^0_u(x) \}$. For these estimates to be useful for decision-making, we want to make the assertion that for some small $\nu' >0$, and for $t\in \{0,1\}$, we have false coverage rate (FCR) bounded by $\nu'$,
\begin{align}\label{eq:fcr}
    \text{FCR}_{\boldsymbol{f^t}} := \Pr_{X, Y(t)}\bigg[Y(t) \not\in [f^t_l(X), f^t_u(X)] \bigg] \leq \nu'~.
\end{align}
Without loss of generality, we will focus on estimating a lower bound for the outcome under treatment $T=t$, meaning we will focus on finding some  $f^t_l(x)$ such that for a small $\nu >0$, we have that
\begin{align}\label{exp:probf1}
    \Pr_{X, Y(t)}[f^t_l(X) \leq Y(t)] \geq 1 - \nu. 
\end{align}
Note that in expressions~\ref{eq:fcr} and~\ref{exp:probf1} the probabilities are defined over $p(X, Y(t)) \neq p(X, Y\mid T=t)$, due to confounding. However, under Assumption 1, this probability is identifiable from observed data. 

It will be useful to restate our objective in terms of the (signed) residual of a function $f$, defined next. 
\begin{thmdef} For an arbitrary function $f$, the signed residuals for $x,y \in \cX \times \cY$:
$
    \underline{r}_f(x, y) = y - f(x).
$
\end{thmdef}
Expression~\eqref{exp:probf1} can be restated as $\Pr[\underline{r}_{f^t_l}(X, Y(t)) \geq 0] \geq 1-\nu$. To be more cautious, we might wish to leave a ``buffer zone'' or a margin, and instead demand that $\underline{r}_{f^t_l}(x, y) \geq \gamma$ for some $\gamma >0$. In this setting, a violation occurs when $\underline{r}_{f^t_l}(x, y) < \gamma$. Larger values of $\gamma$ would imply higher reliability: we are more confident that we are unlikely to observe a violation of the bounds, i.e., unlikely to overestimate the outcome under treatment $t$. With that, direct parallels could be drawn between our setup and that of maximum-margin algorithms: we want to ensure that the signed residual is larger than 0 by a margin of $\gamma$. The larger $\gamma$ is, the more confident we are that our lower bound holds. We can now define the unobserved risk that we wish to study: 
\begin{thmdef}
For $f^t_l \in \cF$, $\gamma >0$, we define the risk of overestimation over the full unknown distribution: 
\begin{align*}
   \underline{R}_{f^t_l}(\gamma) & = \mathbb{E}_{X, Y(t)}\left[ \indic \{\underline{r}_{f^t_l}(X,Y(t)) < \gamma \} \right]. 
\end{align*}
\end{thmdef}
To account for confounding due to biased (non-randomized) treatment assignment, we consider a re-weighted risk:
\begin{align*}
   \underline{R}^w_{f^t_l}(\gamma) & = \mathbb{E}_{X, Y \mid T}\left[   w(x)\indic \{ \underline{r}_{f^t_l}(X, Y) < \gamma \} \mid T=t \right]
\end{align*} 
Under Assumption~\ref{assump:id}, $\underline{R}_{f^t_l}(\gamma) = \underline{R}^{w_1}_{f^t_l}(\gamma)$. Since our notions of confidence are closely related to the margin, $\gamma$, it will be more useful to reason about the magnitude of margin violations, which is defined next. 
\begin{thmdef} 
For $ z = \{ x_i, y_i\}_{i: t_i = t}$, where $x_i, y_i \sim p_t(X, Y)$, known $\boldsymbol{w}_t$, $f^t_l \in \cF$, and $\gamma >0$, we define the average weighted magnitude of training set violations as 
    \begin{align*}
                \underline{D}^{\boldsymbol{w}_t}(z, f^t_l, \gamma) = \sum_{x,y \in z}  w_t(x) \max\{0, \gamma - \underline{r}_{f^t_l}(x,y) \} 
    \end{align*}
\end{thmdef}

In the remainder of this section, we give bounds on expected margin violation as a function of $\underline{D}^{\boldsymbol{w}_t}$. We restrict our analyses to sturdy function classes, as defined in \cite{st_williamson_99} with with range = $[a, b]$. Informally, sturdy function classes have images that are compact subsets of $\mathbb{R}$. We rely on the covering number as a measure of complexity of the analyzed function classes. We use fat-shattering dimensions to study how fast the complexity of a function class can grow with the sample size. Explicit definitions of these three concepts are presented in the supplement (definitions~\ref{def:sturdy}, \ref{def:cover} and~\ref{def:fat} respectively). 

\subsection{Generalization of reliable estimators}
We start by studying the risk of overestimation for re-weighted estimators. To make our main finding easy to follow, we focus on the class of linear functions in a kernel defined feature space. Theorem~\ref{thm:soft_margin_general} in the supplement gives the analogous bounds for more general function spaces. 

\begin{thmthm}\label{thm:soft_margin}
Let $\cF$ be the class of linear functions in a kernel defined feature space, $ z = \{ x_i, y_i\}_{i: t_i = t}$, where $x_i, y_i \sim p_t(X, Y)$, and $C_t$ be as defined in expression~\eqref{eqn:cortes_c}. For $f^t_l \in \cF$, and any $\gamma>0$, let the associated $\underline{D}^{\boldsymbol{w}_t}(z, f_t^l, \gamma ) = D >0$. With a probability $ 1-\delta$ over the draw of random samples, we have that: 
\begin{equation}
\begin{aligned}
  \underline{R}_{f^t_l}(\gamma) & \leq \frac{4 C_t (k_t + \log\frac{1}{\delta} )}{3n_t} + \sqrt{\frac{8 d_2(p||p_t) ( k_t + \log\frac{1}{\delta})}{n_t}}
 \end{aligned}\label{eq:thm1}
\end{equation}
where, for $t\in \{0,1\}$,
\begin{align*}
   k_t  & = \bigg\lceil \log \cN(\sfrac{\gamma}{2}, \cF, 2n_t)  + \frac{D}{\gamma} \log \frac{\exp(n_t + \sfrac{D}{\gamma} -1)}{\sfrac{D}{\gamma}} \bigg\rceil~.
\end{align*}
\end{thmthm}
The proof is outlined in the supplement. \textbf{Remarks:} %

1. Theorem \ref{thm:soft_margin} states that the expected rate of overestimation is bounded by terms at most linear in $k_t$---the sum of the log covering number of $\cF$ \emph{as defined by the margin $\gamma$}, and the ratio of the violations on the training data to $\gamma$. The fact that the covering number is controlled by the margin parameter $\gamma$ shows that the complexity of this learning task relies on how certain we wish to be that the lower bound is not overestimated; more certainty requires a larger $\gamma$ which implies a smaller log covering number. This approach departs from previous literature which instead shows that the sample complexity of risk minimization relies on the covering number of a class containing the true function~\citep{alaa18}. In applications where it is sufficient to have reliable \emph{bounds} on the potential outcomes to make good decisions, this finding can be crucial---especially if the outcomes are difficult to estimate accurately using small samples. Note that the covering number can be bounded by the fat-shattering dimension at a scale proportional to $\gamma$. 

2. Both terms in $k_t$ decrease as $\gamma$ increases, which means that the risk of overestimation decreases as $\gamma$ increases. This property is important because it implies that we can control the risk of overestimation by requiring a large margin.  To see that, note that larger $\gamma$ shrinks the space of viable functions, which decreases the $\gamma$-covering number. The second term includes the ratio of the sum of violations on the training set, $D$, which decreases as $\gamma$ increases, to $\gamma$. Hence the second term also decreases as $\gamma$ increases.

Corollary~\ref{col:ite_fcr} in the supplement, builds on theorem~\ref{thm:soft_margin} to get a bound on the generalization error for bounds on the ITE.

\subsection{Generalization of reliable, informative estimators}
 Theorem~\ref{thm:soft_margin} establishes that the probability of overestimation decreases as we increase the margin $\gamma$. However, arbitrarily large values of $\gamma$ could result in excessively ``cautious'' estimates with low risk of overestimation, at the expense of being too loose to be useful in guiding decisions. In this work, we consider bounds to be informative or have high utility if they imply low uncertainty in the value of the true potential outcomes. We restrict ourselves to definitions of uncertainty that rely on the interval width (IW) of bounds $\boldsymbol{f} := (f_u, f_l)$
\begin{equation}\label{eq:iw}
    \text{IW}_{\bsf}(x) := f_u(x) - f_l(x)~.
\end{equation} 
Smaller IW$_{\bsf}(x)$ implies that bounds are tighter, which implies less uncertainty in the value of the potential outcomes. Intuitively, for $f_u$ and $f_l$ to give small IW$_{\bsf}$, they need to \emph{close} to each other. We define these ``close'' functions and the classes to which they belong as follows: 

\begin{thmdef}\label{def:lb_class}
    Let $p \geq 1$, and $\cX := \{ x: ||x|| \leq r \}$. 
    We say that two classes of bounded linear functionals $\cF_l, \cF_u$ are \emph{informative} if  $\cF_l \subseteq \{ \cX \ni x \mapsto \langle f_l,  x \rangle, ||f_l|| \leq A\}$ and  $\cF_u \subseteq  \{\cX \ni x \mapsto \langle f_u,  x \rangle,  \forall f_l \in \cF_l ; ||f_u - f_l|| < B, \forall x \in \cX : f_l(x) \leq f_u(x)\}.$
\end{thmdef}{}
In words, $\cF_l$ is the set of functions with norm $\leq A$, while $\cF_u$ is the set of functions that are close to functions in $\cF_l$, specifically, within $B$ distance from each $f_l \in \cF_l$. In addition, we specify that $f_l(x) \leq f_u(x)$ for every $x \in \cX$. 

The next theorem extends theorem~\ref{thm:soft_margin} to these informative function classes, allowing us to study the risk of overestimation for tight intervals. To improve readability, log terms which do not affect the interpretation of the statement have been suppressed. The full statement is presented in Theorem~\ref{lem:soft_margin_fat_supp}. 
\begin{thmthm}\label{thm:soft_margin_fat}
Let $\cF^t_l$, $\cF^t_u$, $A$, $B$, and $r$ be as defined in definition~\ref{def:lb_class}, $z$, and $D$ as defined in theorem~\ref{thm:soft_margin},and $C_t$ be as defined in expression~\eqref{eqn:cortes_c}. For $f^t_l \in \cF^t_l$, $f^t_u \in \cF^t_u$ and any $\gamma>0$, with a probability $1-\delta$ over the draw of random samples, 
the bound \eqref{eq:thm1} in Theorem~1 applies with, for $t\in \{0,1\}$,
\begin{align*}
   k_t  & \approx \bigg\lceil \bigg(\frac{r(A+B)}{\gamma}\bigg)^2  + \frac{D}{\gamma} \log \frac{e(n_t + \sfrac{D}{\gamma} -1)}{\sfrac{D}{\gamma}} \bigg\rceil~.
\end{align*}
\end{thmthm}
Theorem~\ref{thm:soft_margin_fat} gives us an idea of how to learn informative bounds that reliably cover the potential outcomes. It suggests that one way to reduce generalization error is to minimize  $A$, the norm of $f^t_l$, $B$ the  distance between $f^t_l$ and $f^t_u$, and $D$, the sum of violations on the training data.

%% file: contents/alg.tex

\section{Learning reliable, informative bounds}\label{sect:algorithm}
We present the Bounded Potential outcomes algorithm (BP) for learning informative bounds on potential outcomes under the constraint that they are violated with low probability. The algorithm is flexible in that it can maximize different utilities or notions of informativeness that the decision maker might have. For brevity, we focus on utility as defined by small IW. BP leverages our theoretical findings by explicitly constraining the violations on the training data, and minimizing some loss function, $\ell$, of the interval widths. 

The appropriate loss function will vary between applications. We consider optimizing three loss functions of IW over $p(x)$: $\ell^{(1)}$ represents the desire to achieve a tight prediction bound on average, captured in the mean absolute interval width. $\ell^{(2)}$ penalizes the mean squared interval width, placing a higher penalty on points with very wide bounds. The third $\ell^{(\infty)}$ minimizes the worst (widest) interval by penalizing the maximum interval width. 

We consider learning under the following conditions. Let $\phi : \cX \rightarrow \mathbb{R}$ be the feature map corresponding to a reproducing kernel $k(x_i, x_j) = \langle\phi(x_i) , \phi(x_j)\rangle$. For treatments $t\in \{0,1\}$ and bounds $b \in \{l,u\}$ (lower/upper), let $f^t_b(x_i) := \langle \theta_b^t , \phi(x_i)\rangle + \rho_b^t$. In this setting, all three losses ($\ell^{(1)}, \ell^{(2)}, \ell^{(\infty)}$) are convex in $\theta$. Let sample weights $w_{t_i}$ be defined as in Definition~\ref{def:pscore}, and define $\widetilde{w}_{t_i} := w_{t_i}/\sum_{j : t_j=t_i} w_{t_j}$. 
Finally, let $\Lambda(f)$ denote a term that measures  complexity of $f$, e.g.,  the squared norm of parameters.

We describe two versions of BP: BP-D, a decoupled version where the bounds for the treated and control groups are fitted separately, and BP-C, a coupled version where the two are fitted simultaneously.
\subsection{BP-D: decoupled treatment groups}\label{sect:decop_alg}
First, we consider estimating bounds $f_u, f_l$ on a single potential outcome $Y(t)$, independently of others. We minimize the weighted loss $ \ell^{(p)}_{\widetilde{w}} (\bsf)$ and desire for bounds to be violated only with small probability over $p(x)$. We let the loss $\ell^{(p)}_{\widetilde{w}}(\bsf)$ be defined by either the mean absolute interval width, $\ell^{(1)}_{\widetilde{w}}(\bsf) \sum_{i:t_i=t} \widetilde{w}_{t_i}|\text{IW}_{\bsf}(x_i)|$, the mean squared interval width, $\ell^{(2)}_{\widetilde{w}}(\bsf)= \sum_{i:t_i=t} \widetilde{w}_{t_i}(\text{IW}_{\bsf}(x_i))^2$, or the maximum interval width, $\ell^{(\infty)}_{\widetilde{w}}(\bsf) = \sup_{i:t_i=t}(\text{IW}_{\bsf}(x_i))$.

\begin{equation}\label{eqn:indep_obj_const}%
\begin{aligned}
& \underset{\bsf = \{f_u, f_l\} }{\text{minimize}}
& & \ell^{(p)}_{\widetilde{w}} (\bsf) + \alpha \Lambda(\bsf) \\
& \text{subject to}
& & \sum_{i:t_i=t} \widetilde{w}_{t_i} \max(y_i - f_u(x_i), 0) \leq  \beta_u  \\ 
& & & \sum_{i:t_i=t} \widetilde{w}_{t_i} \max(f_l(x_i) - y_i, 0) \leq  \beta_l \\ 
& & & f_l(x_i) \leq f_u(x_i)~, \forall i : t_i = t~.
\end{aligned}
\end{equation}

Note that the constraints are defined with respect to the magnitude of the violations, which does not immediately  translate into a specific FCR. We address this issue in section~\ref{sect:xval}. Problem \eqref{eqn:indep_obj_const} can be solved separately for the two treatment groups, as is done in two-learners or the treatment variable could be added in as a feature and the two treatment groups can be jointly trained, as is done in single-learners \citep{kunzel2019metalearners}. Next, we highlight some important characteristics of this estimator.\\
\\
\textbf{1. BP-D minimizes the lower bound in Theorem~\ref{thm:soft_margin_fat}.} Note that BP-D is specified over the set of linear functions with kernel defined feature spaces. With $\Lambda$ defined as the 2-norm of the vector $\theta$, and because of the last constraint ($f_l \leq f_u$), the functions returned by BP-D fall within the set of functions defined in definition~\ref{def:lb_class} with high probability, and hence theorem~\ref{thm:soft_margin_fat} is applicable here. Recall that theorem~\ref{thm:soft_margin_fat} states that for this estimated function to be optimal, they need to minimize $A=||\theta||, B = $ distance ($p-$norm) between the upper and the lower bounds and $D = $ the magnitude of the training set violations while maximizing $\gamma$. Problem~\eqref{eqn:indep_obj_const} directly minimizes the $A, B$ (for $p=1, 2, \infty$ depending on $\ell$) and $D$. As for $\gamma$: suppose we fix the bias to be $\tilde{\rho}^t_b$, then $\gamma^t_b =  \tilde{\rho}^t_b - \rho^t_b$, where the latter is the bias returned by solving problem~\eqref{eqn:indep_obj_const}. Because problem~\eqref{eqn:indep_obj_const} minimizes $\rho^t_b$, it maximizes $\gamma^t_b$ for a fixed $\tilde{\rho}^t_b$. Ideally, we would not fix $\tilde{\rho}^t_b$ in advance, but let it be decided by the data. We address this issue in section~\ref{sect:xval}. \\
\\
\textbf{2. BP-D with $\ell^{(1)}$-loss is equivalent to quantile regression.} When minimizing the mean absolute interval width, our problem reuces to a quantile regression with non-crossing constraints~\citep{takeuchi2006nonparametric} of quantiles $q$ and $1-q$ for for some choice of $q \in (0, .5)$.
\begin{thmthm}\label{thm:quantile}
Assume that \eqref{eqn:indep_obj_const} is strictly convex and has a strictly feasible solution. Then, for any fixed quantile $q \in (0.5, 1)$, there are parameters $\beta_u, \beta_l \geq 0$ such that the minimizers $f_u^*, f_l^*$ of \eqref{eqn:indep_obj_const} with absolute loss and the minimizers of the quantile loss for quantiles $(q, 1-q)$, with non-crossing constraints, are equal.
\end{thmthm}
A proof is given in the appendix.

BP-D allows us to learn reliable and informative bounds but it does not make use of the ``unlabeled'' data from the opposite treatment group. This is addressed next. 
\vspace{-.5em}

\subsection{BP-C: coupled treatment groups}\label{sect:main_alg}
In the coupled problem, we make use of samples from the counterfactual treatment group in two ways. First, we apply constraints that ensure that the lower and upper bounds do not cross also for counterfactual outcomes. Second, the loss functions are defined with respect to the full marginal distribution of subjects (including counterfactual treatment assignments). We define the coupled version of the mean absolute loss $\ell^{(1)} =  \sum_{i=1}^n \sum_{t=0}^1 \widetilde{w}_{t_i}|\iw_{\bsf^t}(x_i)|$, mean squared interval width, $\ell^{(2)} =  \sum_{i=1}^n \sum_{t=0}^1 \widetilde{w}_{t_i}\iw_{\bsf^t}(x_i)^2$, and maximum interval width,  $\ell^{(\infty)} = \sup_{i=1}^n \sum_{t=0}^1\iw_{\bsf^t}(x_i)$. The coupled problem becomes:
\begin{align}
& \underset{\{\bsf^t = \{f^t_u, f^t_l\}\} }{\text{minimize}}
& & \ell^{(p)}_{\widetilde{w}}(\bsf^0, \bsf^1) + \alpha \cdot (\Lambda(\bsf^0) + \Lambda(\bsf^1)) 
\nonumber \\
& \text{subject to}
& & \sum_{i:t_i=t} \widetilde{w}_{t_i} \max(y_i - f^t_u(x_i), 0) \leq  \beta_u,  \forall t \nonumber \\ 
& & & \sum_{i:t_i=t} \widetilde{w}_{t_i} \max(f^t_l(x_i) - y_i, 0) \leq \beta_l, \forall t \nonumber \\ 
& & & f_l^t(x_i) \leq f_u^t(x_i)~, \forall t, i :t_i = t~. \label{eqn:coupled_obj_const}
\end{align}

Given Assumption~\ref{assump:id}, specifically, the assumption of overlap this encourages the counterfactual outcome intervals to be small even if the corresponding treatment assignment is not observed. By coupling the two objectives, we allow information to be shared between the treated and non-treated populations in a semi-supervised way.  We caution, however, that in the absence of overlap, the coupled loss might be overly optimistic about in regions of non-overlap, returning intervals that do not cover the true data. With $f_l, f_u$ linear in the representation $\phi$ and $\Lambda(f)$ defined as the L2 norm of the function weights, expressions~\eqref{eqn:indep_obj_const} and~\eqref{eqn:coupled_obj_const} are both convex programs which can be readily solved by a general solver. Our code is available at \href{https://github.com/mymakar/bpo.git}{$<$\texttt{github.com/mymakar/bpo.git$>$}}.

\subsection{Cross-Validating BP}\label{sect:xval}
BP-C/D requires a regularization parameter, $\alpha$, a level of tolerance to violations, $\beta_{u,l}$, and $\sigma$, which controls the kernel (e.g., the length scale for Gaussian kernels or the polynomial degree for polynomial kernels). Suppose that we solve problem~\eqref{eqn:indep_obj_const} or~\eqref{eqn:coupled_obj_const} and get some estimate for the bias $\tilde{\rho}^t_b$, we specify an additional parameter $\gamma >0$, and take the final estimate $\rho^t_l := \tilde{\rho}^t_l - \gamma$ and $\rho^t_u := \tilde{\rho}^t_u + \gamma$. This allows us to set $\gamma$ based on the data rather than specify it apriori. 

BP constrains the magnitude of the violations rather than the FCR directly. This allows the algorithm to directly reflect the theory and makes the optimization problem easier. The disadvantage is that the magnitude of violations does not directly translate into a specific FCR. We address this issue by designing a cross-validation algorithm that picks the hyperparameters of the model to achieve a required FCR, $\nu$. The algorithm takes as an input the training data, $\nu$, $\ell$, the required loss to minimize, and $M$, the set of hyperparameters to consider. We then split the data into training and validation. For each set of parameters $m \in M$, we use the training set to solve problem~\eqref{eqn:indep_obj_const} or~\eqref{eqn:coupled_obj_const}. We estimate $\hat{\nu}_{m}$ and $\widehat{\ell}_{m}$, the FCR and loss corresponding to $m$ on the held-out set. We discard of all the hyperparameters with a corresponding $\hat{\nu}_{m} > \nu$, and define $M' = \{m: \hat{\nu}_{m} \leq \nu \}$. We set the optimal hyperparameters $m^* := \min_{m \in M'} \widehat{\ell}_{m}$. The procedure is summarized in Algorithm~\ref{alg:xval} in the supplement.

%% file: contents/experiments.tex
\section{Experiments}

We compare our model to other interval estimation methods. First is classical confidence-interval based approaches. We use \textbf{XX-CCI} to refer to this approach, where XX will be replaced by the name of the base model (e.g., if it is a Gaussian Process, we use GP-CCI). While popular, confidence intervals are known to have poor coverage in finite samples \cite{sargent1992investigation,lei2018distribution}. Conformal intervals, the second interval estimation method we compare against, were introduced as an alternative with better finite sample coverage \cite{lei2018distribution}. Conformal intervals are estimated by splitting the training data into two parts. The first part is used to train the outcome model, where parameters are picked via the usual cross-validation techniques. We estimate the residuals on the second subset of the training data. If the required FCR is $q$, we take the $1-q^{th}$ quantile of the residuals to be a ``shifting'' parameter (akin to $\gamma$ in our setting). The conformal intervals for a test sample are taken to be the estimated outcome $\pm$ the shifting parameter. We use \textbf{XX-CI} to refer to this approach. Finally, we introduce $\gamma$-intervals, which we refer to as \textbf{XX-}$\boldsymbol{\gamma}$. Similar to conformal intervals, we split the data into two, fitting the best model on the first half and then picking the smallest shifting parameter $\gamma$ that achieves the required FCR on the second half. We use \textbf{BP-V-Lp} to refer to our models, where V refers to the D (decoupled) or C (coupled) version and Lp refers to the norm of the loss (1, 2, or $\infty$). Recall that the 1-norm is similar to quantile regressions (QR) (by theorem~\ref{thm:quantile}). 

We evaluate the performance of our models and the baselines on a held-out test set with respect to two criteria: the achieved FCR, as defined in equation~\eqref{eq:fcr} and the utility as measured by the mean IW and the max IW, as defined in equation~\eqref{eq:iw}.  Additional cross-validation details for our model and the baselines are included in the supplement. 

We analyze settings where we expect BP to outperform baselines. Most baselines make restrictive assumptions about the distributions of the residuals. When such assumptions break, the resulting intervals are no longer tight or do not correctly cover the outcomes. We briefly outline such assumptions: 
\begin{enumerate}[topsep=0pt,itemsep=-0.5ex,partopsep=1ex,parsep=1ex,leftmargin=0pt]
    \item \textbf{Symmetry}. This assumption states that in order to get a 5\% FCR, we need to ensure that the lower and upper bounds are violated by at most 2.5\% each. In some cases, the tightest bounds would be achieved by non-symmetrical bounds, e.g., the lower bound is violated by 1\% whereas the upper bound is violated by 4\%. Violations to the symmetry assumption occur, for example, when the model is misspecified, which leads to biased estimates. In that case, tight bounds should reflect the direction of bias: if the estimates are biased downwards (meaning lower than the true value), it is more important that the upper bounds are not violated, whereas violations to the lower bound are more permissible (since the estimate itself is a lower value than the true outcome). 
    \item \textbf{Well-behaved residual distribution}: This assumption states that the residuals concentrate around a single, central value. Such an assumption is also violated when there is model misspecification, or if the outcome noise is heteroskedastic.
\end{enumerate}
We stress that our approach does not make these assumptions. Our analysis will focus on setting where violations to the symmetry assumption might occur. Additional analysis in section~\ref{sect:ist_hsk} in the supplement shows the our approach is superior when the well-behavedness assumption is violated (in the presence of heteroskedasticity). In addition, section~\ref{sect:acic_1k} in the supplement includes shows that in settings where the two assumptions are unlikely to be violated, BP still outperforms other kernel-based methods. 

\subsection{IST data}\label{sect:ist}
\begin{figure*}
  \begin{minipage}{.51\textwidth}
  \captionsetup[subfigure]{justification=centering}
    \begin{subfigure}{\columnwidth}
        \centering
        \caption{ Comparing different loss functions \label{fig:ist_y1}}
        \includegraphics[width=.75\textwidth]{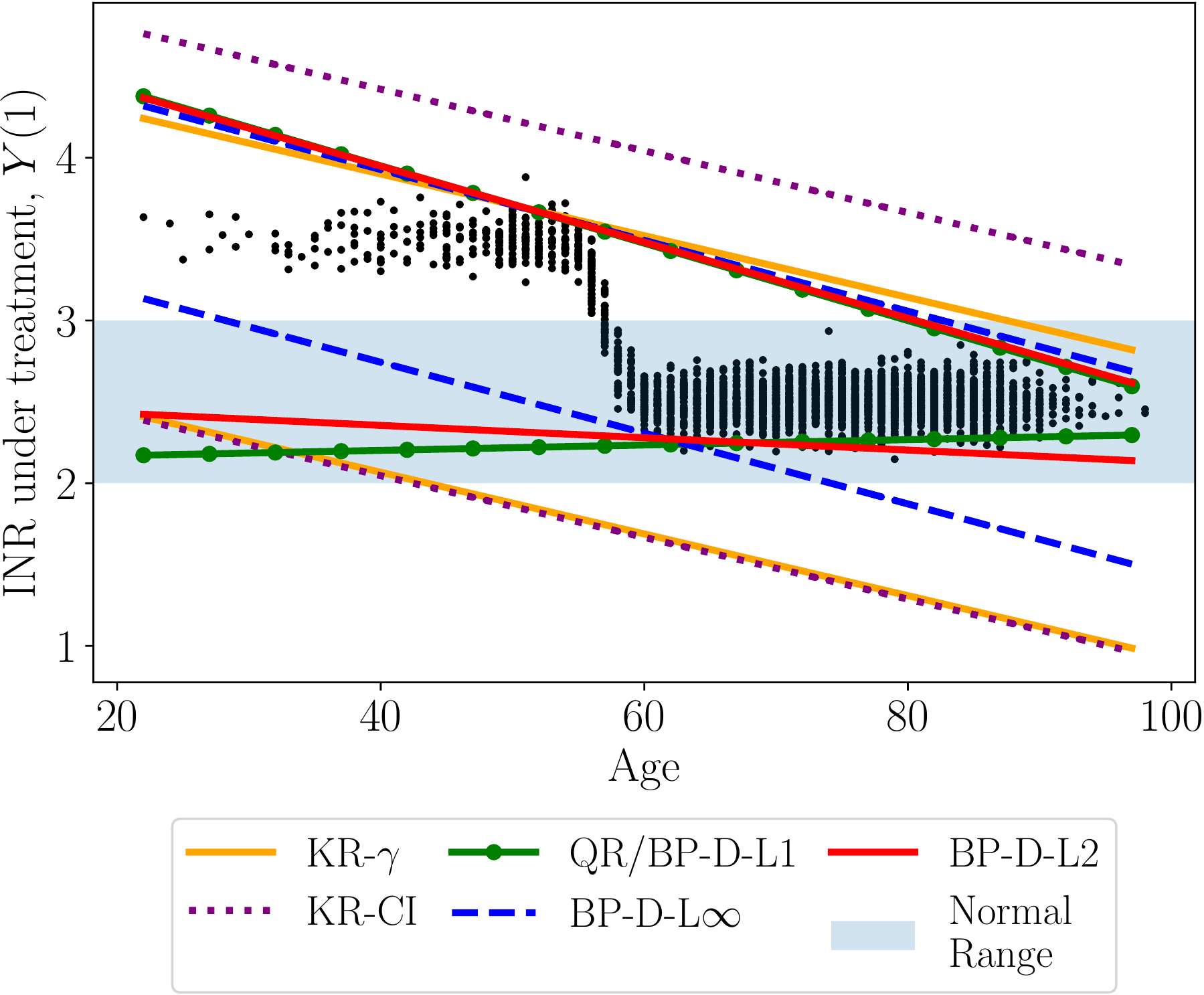}\vspace{-.5em}%
    \end{subfigure}
  \end{minipage}
  \begin{minipage}{.49\textwidth}
  \captionsetup[subfigure]{justification=centering}
    \begin{subfigure}{\columnwidth}
        \centering
        \vspace{-1em}
        \caption{Decoupled and coupled versions\label{fig:ist_y0}}
        \includegraphics[width=.8\linewidth]{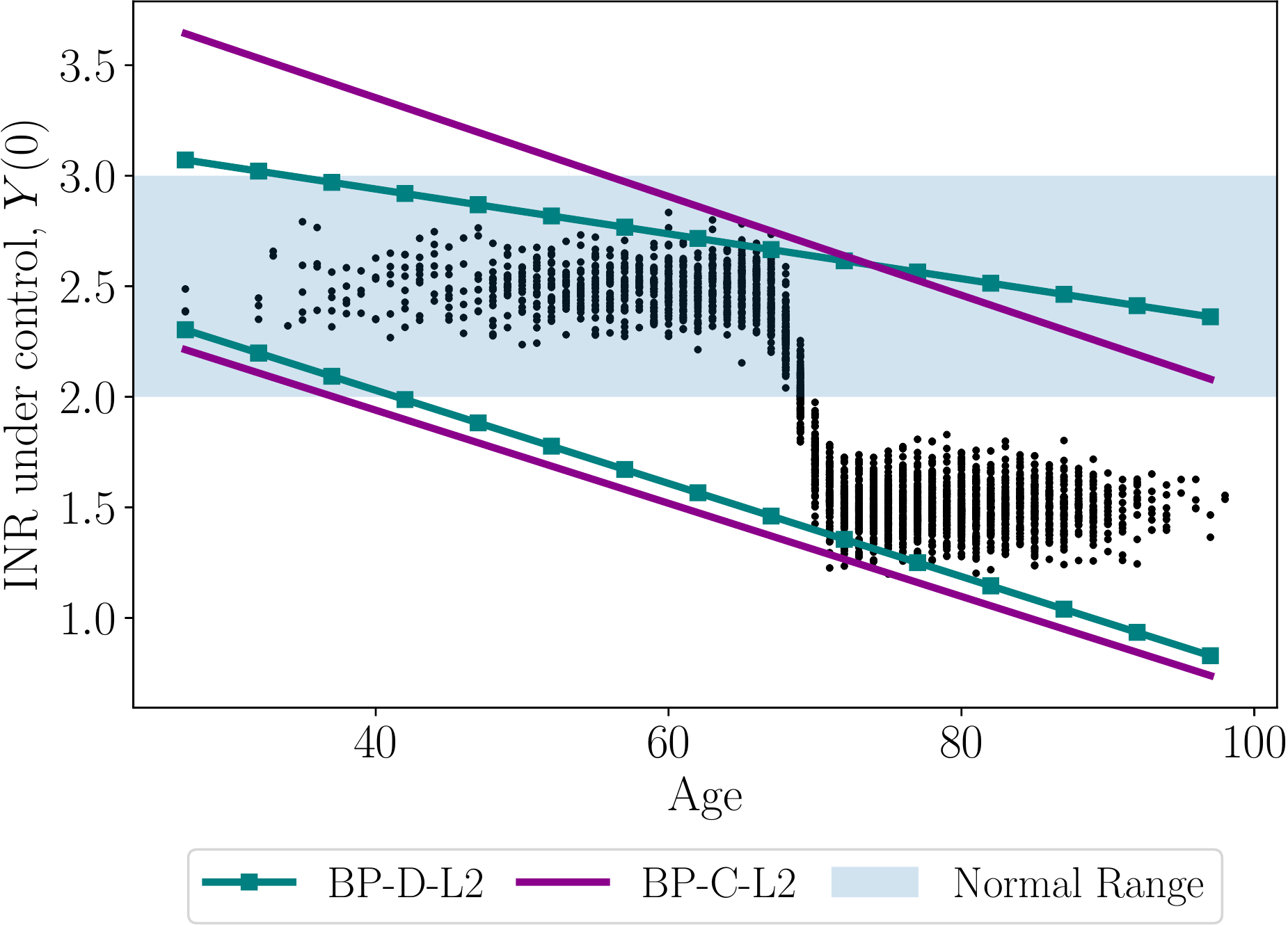}%

    \end{subfigure}
  \end{minipage}

  \caption{IST results. Plots show results from a single simulation. Black dots show potential outcomes on the test set, lines show fitted values, and shaded region shows healthy range. 
  Plot~\ref{fig:ist_y1} show that BP-D-L$\infty$ is a ``fair'' objective, ensuring that the younger ($\leq 60$) population has tight intervals, sacrificing tight intervals for older population. QR (equivalent to BP-D-L1) ensures intervals are tight for older population but returns wider intervals for the younger population. BP-D-L2 gives an estimate ``in-between'' the two objectives, penalizing large intervals more aggressively than QR/BP-D-L1. Baselines (KR-CI/KR-$\gamma$) return bounds that are loose for both populations. Plot~\ref{fig:ist_y0} shows that penalizing the counterfactual interval widths enables the coupled objective, BP-C-L2, to return a tighter fit for $Y(0)$ in the area where few untreated examples exist in the training data (age$>70$). 
\vspace{-1.2em} \label{fig:ist}}
\end{figure*}

\begin{table}
\centering
\caption{IST results. Table shows results averaged over 20 simulations, confirming conclusions from figure~\ref{fig:ist}. \label{tab:ist}}
\vspace{.1in} 
\resizebox{.8\columnwidth}{!}{%
\begin{tabular}{l|lll}
\toprule
Model &          FCR &       Mean IW &        Max IW \\
\midrule
\multicolumn{4}{c}{$Y(1)$ results}\\
\midrule
BP-D-L2         &  0.007 (0.36)  &  1.04 (0.05)  &  2.15 (0.19)  \\
BP-D-L$_\infty$ &  0.007 (0.37)  &  1.16 (0.06)  &  1.16 (0.06)  \\
QR/BP-D-L1         &  0.007 (0.43)  &  1.07 (0.09)  &  2.25 (0.26)  \\
KR$-\gamma$     &  0.004 (0.81)  &  1.96 (0.09)  &  1.96 (0.09)  \\
KR-CI           &   0.0 (0.0)  &  2.41 (0.07)  &  2.41 (0.07)  \\
\midrule
\multicolumn{4}{c}{$Y(0)$ results}\\
\midrule
BP-C-L2         &  0.007 (0.59)  &  1.35 (0.17)  &  1.62 (0.26)  \\
BP-D-L2         &  0.005 (0.51)  &  1.37 (0.13)  &   1.72 (0.2)  \\
\bottomrule
\end{tabular}}%
\vspace{-1em}
\end{table}

We begin with a simple illustrative example that highlights the strengths of BP vis-a-vis baselines and the properties of different utility functions in a practical setting. We aim to answer the following: (1) How do different losses reflecting different notions of utility affect the estimates? (2) How does the coupled objective make use of counterfactual data? 

We study the task of a physician deciding whether or not to prescribe Heparin, an anticoagulant, to reduce the risk of Ischemic and Hemorrhagic strokes. Patients with an elevated risk of forming blood clots can reduce their risk of an Ischemic stroke by taking Heparin. However, some patients experience excessive bleeding if placed on Heparin increasing their risk of a Hemorrhagic stroke. In this setting, to make an informed decision, the physician only needs to know if the INR under treatment \emph{roughly} falls within the healthy range of 2--3 as described in the introduction. The exact value of INR provides little additional insight. 

We use data from a randomized control trial measuring the effects of Heparin \citep{group1997international}. We restrict our analysis to the patients who received Heparin (treatment, $n_1 = 4530$) or no anticoagulant (control, $n_0 = 4534$).  To introduce confounding, we drop 70\% of the older (age $>70$), untreated population. Note that the distribution of age in the trial is skewed, with a mean of 71.8 and a skewness of -0.79, which means that young patients are under-represented.  Figure~\ref{fig:hist} in the supplement shows the distribution of ages for the treated and control groups in the training set. 
Because INR was not measured in the original data, we simulate the INR under treatment according to $\mathbb{E}[Y_i(1)\mid age_i] = S(-5, age'_i) + 2.5$, where $S(a, x)$ denotes the sigmoid function with coefficient $a$, and $age'$ is the age rescaled between -10, 10. This setup ensures that the majority of the population (older than 60) falls within the normal range if treated, while the few young patients younger than 60 have high INR if treated. Similarly, the outcome under control is determined by $\mathbb{E}[Y_i(0)\mid age_i] = S(-5,age'_i-4) + 1.5$. This reflects the setting where patients older than 70 (who are under-represented in the untreated population) would have too low of an INR if not placed on Heparin. Noise for both $Y(1)$ and $Y(0)$ is drawn from a Gaussian distribution with mean 0 and variance 0.1. 

We assume that the physician is restricted to linear models. In this setting the models are inherently misspecified, which means that the residuals violate the symmetry and well behaved-ness assumptions. We fit a kernel regression with a linear kernel (\textbf{KR}) for the baselines. We repeat our simulation 20 times and report averages. In each simulation, we randomly sample 3000 patients for training and validation and 3000 held out for testing. Following \citet{chernozhukov2016double}, we use half the training data to estimate the nuisance parameter, that is the propensity scores, and the other half to fit the potential outcomes. For propensity scores, we fit a logistic regression. We pick the regularization parameter for the propensity score model and all the response surface models via 3-fold cross-validation as described in detail in the supplement. For all experiments, we set the required FCR to be $\leq 0.01$, i.e., $ \leq 1\%$.

Table~\ref{tab:ist} (top) shows that BP-D-L$\infty$ achieves the smallest max IW. BP-D-L2 and QR (equivalent to BP-D-L1) achieve the smallest mean IW, with the former achieving a smaller max IW. Figure~\ref{fig:ist_y1} explains why. BP-D-L$\infty$ achieves the smallest max IW since it penalizes large intervals in the younger population while sacrificing by fitting a wider interval for age $\geq 60$. Such an objective is most appropriate when issues of fairness might be at play, such as if a physician wants to ensure that younger patients are never given abnormally large intervals compared to the older group. 
QR/ BP-D-L1 achieves a tight mean IW for the older population but sacrifices for the younger population. Such an objective is appropriate when we want estimates that are as tight as possible on average, even if that entails computing wide estimates for small subpopulations. BL-D-L2 is in between the two extremes of BP-D-L$\infty$ and BP-D-L1/QR; its mean IW is slightly higher than that of BP-D-L1 (for the younger population) and lower than that of BP-D-L$\infty$, its max IW is lower than that of BP-D-L1 but higher than that of BP-D-L$\infty$. This is because the L2 loss penalizes large IWs more aggressively than L1. Most notably, KR-CI and KR-$\gamma$ return loose estimates compared to BP/QR. This is because KR-CI  assumes symmetry of the residuals, returning overly loose upper bounds. KR-$\gamma$ implicitly assumes non-fat tailedness by shifting the estimates by the same constant for all individuals. More generally, the baselines fail because they aim to first estimate the outcome as best as possible, and then estimate the intervals post-training. Ultimately, the model is picked based on what reduces the mean squared error, not what reduces over/under-estimation. 

A physician who prescribes Heparin only when they are certain that a patient's INR would fall in the normal range (i.e., both upper and lower bounds fall in the normal range) would not prescribe heparin to anyone if they rely on KR-$\gamma$, KR-CI, or BP-D-L$\infty$ estimates. The latter has the advantage of providing tighter bounds for the younger patient group, whereas the former three also fails on that task.

Table~\ref{tab:ist} (bottom) shows that the decoupled version achieves a smaller mean and max IW compared to the coupled version, though the difference is not statistically significantly different. Figure~\ref{fig:ist_y0} gives insight into the difference between the two versions. The coupled objective returns tighter intervals for the majority of the population, that is patients with age~$>70$, who are under-represented in the control group. This happens because the coupled objective has an incentive to minimize the interval width for older, untreated patients since wider counterfactual interval for the old treated patients is penalized, whereas the decoupled objective is unaware of these patients. 

\subsection{ACIC data}\label{sect:acic}

Next, we evaluate our approach in a more challenging, high-dimensional task: semi-simulated data from the Atlantic Causal Inference Conference Competition \citep{dorie2017automated}. In this task, 58 variables were extracted from the Collaborative Perinatal Project, a study on pregnant women and their children. The treatment assignment and the response surfaces were simulated. We focus on the simulation with limited overlap and high heterogeneity where the treatment response surface is polynomial and the response surface is exponential. We sample 200 data points for the training/validation of the main models, and 1000 for our test set. We sample 1000 data point for training/validation of the propensity score models. Propensity scores are estimated using 3 fold cross-validation. 
\begin{figure}
  \begin{minipage}{\columnwidth}
  \captionsetup[subfigure]{justification=centering}
    \begin{subfigure}{\columnwidth}
        \centering
        \caption{Comparing tightness of estimated intervals \label{fig:acic0}\vspace{-.5em}}
        \includegraphics[width=.8\columnwidth]{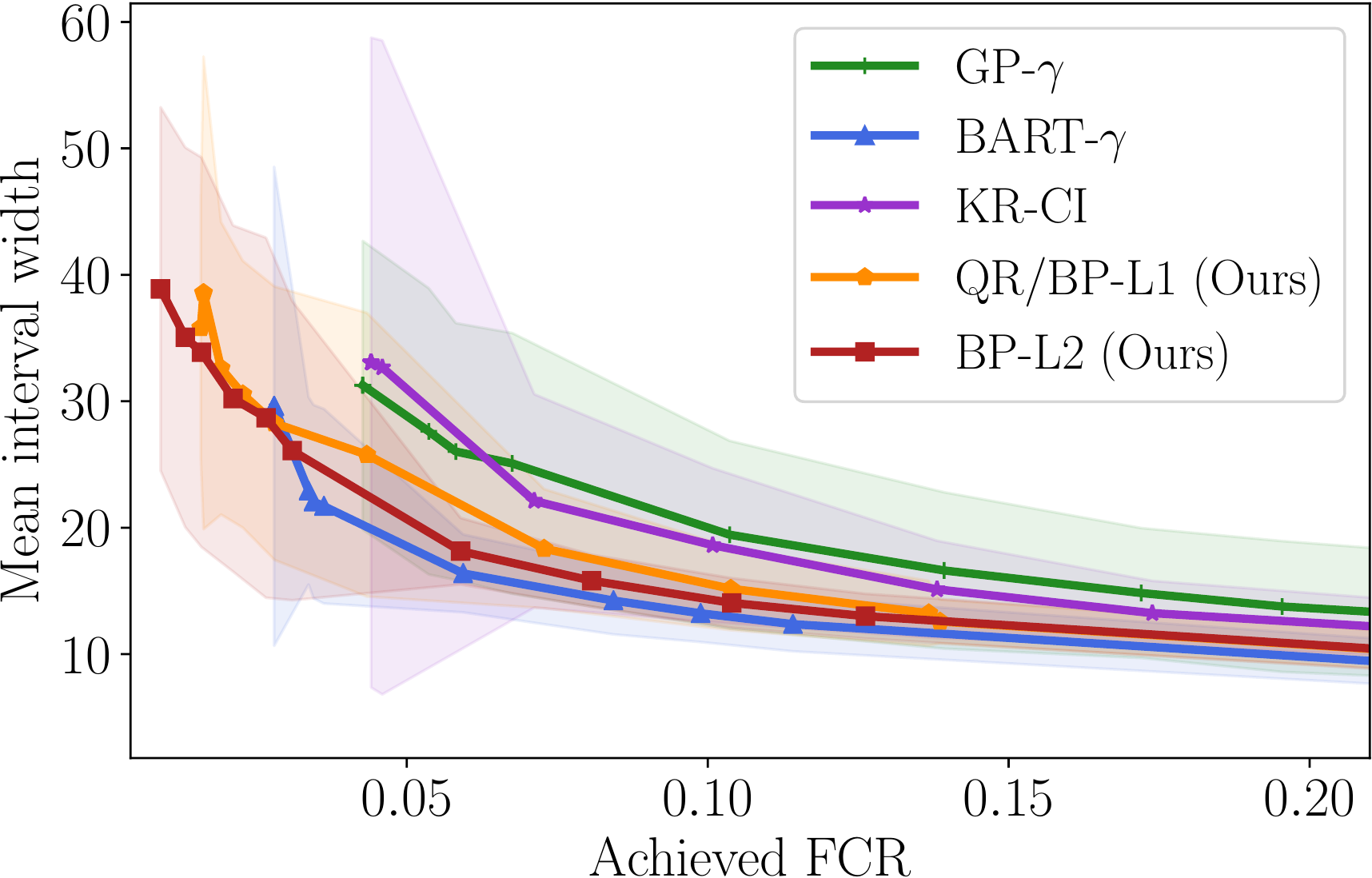}%
    \end{subfigure}
  \end{minipage}
  \begin{minipage}{\columnwidth}
  \captionsetup[subfigure]{justification=centering}
    \begin{subfigure}{\columnwidth}
        \centering
        \caption{Comparing violation to the required FCR\label{fig:acic1}\vspace{-.5em}}
        \includegraphics[width=.8\columnwidth]{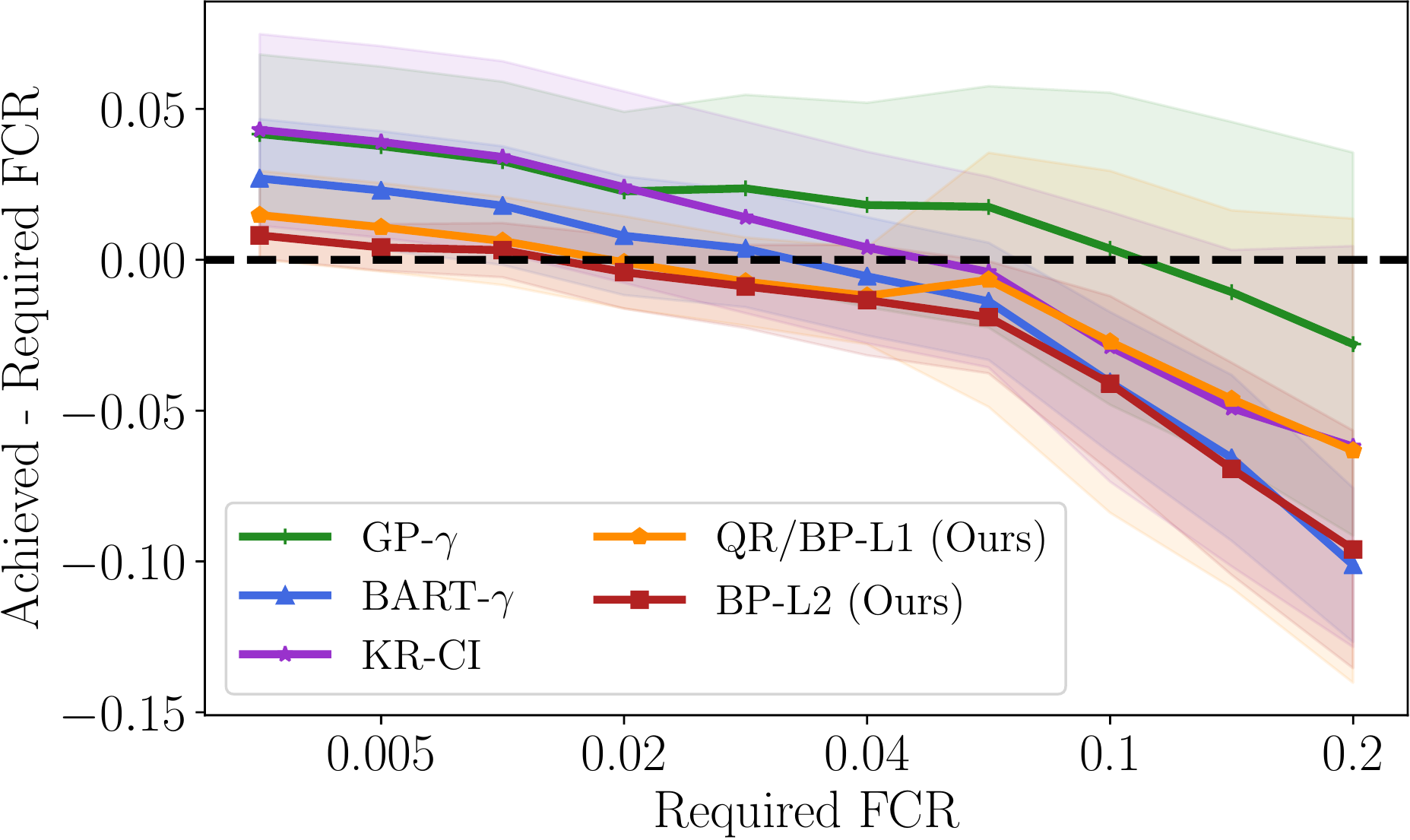}%
    \end{subfigure}
  \end{minipage}
  \caption{ACIC results. Plots show results averaged over 20 simulations. Plot~\ref{fig:acic0} shows the mean interval width for different values of the achieved FCR on a held-out test set. Plot~\ref{fig:acic1} shows the violation of the required FCR (= achieved - required) at different values of required FCR. Models above the dotted black line are in violation of the required FCR. The two plots show that BP achieves a mean interval width comparable to that of BART but at a lower violation of the required FCR. BP outperforms all kernel-based methods in terms of mean interval width and violation to the required FCR.  
\vspace{-1.2em} \label{fig:acic}}
\end{figure}

To fit the potential outcomes, we use an RBF kernel for our BP/QR models. We also use an RBF kernel for the kernel regression models. We only present KR-CI, excluding KR-$\gamma$ since it performs comparably to KR-CI. In addition, we include single-learners \cite{kunzel2019metalearners} with Gaussian processes as the base-estimators (GP), and Bayesian Additive Regression Trees (BART; \cite{hill2011bayesian}). For the latter 2 models, we compute the classical confidence intervals (GP-CCI, and BART-CCI), and a variant of the $\gamma$-intervals (GP-$\gamma$, and BART-$\gamma$). Here, $\gamma$ is used as a scaling rather than a shifting parameter; for an estimated outcome $\hat{y}$, and estimated standard deviation $\hat{\kappa}$, the lower/upper bounds are estimated as: $\hat{y} \pm \gamma \cdot \hat{\kappa}$, and the optimal $\gamma$ is picked based on cross-validation as described previously. 

We focus on getting the tightest bounds, so we only present results from BP-C-L2. We measure the performance of the models at required FCR = \{0.001, 0.005, 0.01,  0.02,  0.03,  0.04,  0.05,  0.1,   0.15, 0.2 \}. 

In this setting, since the small sample size potentially restricts the ability to fit the true functions, which may belong to a complex function class. This can be thought of as a ``forced'' model misspecification since the limited data does not afford us the ability to fit the true function, and limits us to simpler function classes. This is once again, a setting where we expect our models to outperform baselines that make strong assumptions about the residuals. 

Figure~\ref{fig:acic0} shows the mean achieved FCR on the $x-$ axis, and the mean IW on the $y-$axis for our model and baselines averaged over 20 simulations. First, we see that the mean IWs for all the models decrease as the achieved FCR increases. This confirms our theoretical findings that a trade-off between confidence that the bounds cover the potential outcomes and complexity of the function class; lower required FCRs (i.e., higher confidence that the bounds cover the true date) are associated with simpler function classes, which sacrifices accuracy, leading to higher mean IW. Second, we see that our models achieve interval widths that are tighter than all other kernel-based methods, and comparable to BART at every value of achieved FCR. However, figure~\ref{fig:acic1} shows that our models achieve smaller violation compared to BART. This implies that our models are better able to exploit the trade-off between confidence and complexity.

Results from GP-CCI, and BART-CCI are excluded from the plots, and presented in the section~\ref{sect:acic_cci} in the supplement since they achieve very large violations ($\approx$.6 for GP-CCI, and $\approx$.2 for BART-CCI for roughly all required FCRs). This conforms with previous studies that show that CCI methods tend to have poor coverage in finite samples \cite{sargent1992investigation,lei2018distribution}.

%% file: contents/conclusion.tex
\section{Conclusion}
In this paper, we establish that the sample complexity of learning bounds on potential outcomes depends on how confident we wish to be that the bounds cover the true potential outcomes. For applications where it is sufficient to have reliable bounds on the potential outcomes to make good decisions, and the outcomes are complex functions, our findings indicate how to simplify the learning problem. Based on these findings, we introduced an algorithm that maximizes a notion of usefulness, specified by the decision maker, subject to  constraints that guarantee validity of the bounds with high probability. Using semi-synthetic data, we showed that our algorithm can guide physicians in making treatment decisions for stroke patients. We also showed that our method outperforms baselines, estimating tight prediction intervals without violating a required level of false coverage rate.

\section*{Acknowledgements}
We thank the anonymous reviewers, Uri Shalit, Alex D'Amour, and members of the Clinical and Applied Machine Learning group at MIT for insightful suggestions and feedback. We thank Lucas Wittman, Amal Ramsis and Samer Moussa for medical feedback. DS and FJ were supported in part by Office of Naval Research Award No. N00014-17-1-2791. FJ was also partially supported by the Wallenberg AI, Autonomous Systems and Software Program (WASP) funded by the Knut and Alice Wallenberg Foundation. MM was funded by Wistron Corp, Quanta Computers, and Microsoft. JG was funded by Wistron Corp, and Quanta Computers.

%% file: contents/supplement.tex
\section{Additional definitions}
The following definitions will be useful to prove our main statements. 

\begin{thmappdef} \label{def:sturdy}[Restated from \citet{st_williamson_99}] We say that a function class $\cF$ is sturdy if it maps $X$ of size $n$ to a compact subset of $\mathbb{R}^n$ for any $n \in \mathbb{N}$. 
\end{thmappdef}

\begin{thmappdef}\label{def:cover}
    Let $(X, l_\infty)$ be a pseudo-metric space defined with respect to the $l_\infty$ norm, and let $A$ be  a subset of $X$ and $\epsilon >0$. A set $U \subseteq X$ is an $\epsilon$-cover for $A$ if for every $a \in A$, there exists $u \in U$ such that $|| a- u||_{l_{\infty}} \leq \epsilon$. The $\epsilon$-covering number of $A$, $ \cN(\epsilon, A, d)$ is the minimal cardinality of the $\epsilon$-cover for $A$. 
\end{thmappdef}

\begin{thmappdef}\label{def:fat}[Restated from \cite{bartlett1999generalization}]
For $\gamma \in [0, \infty]$, and $\cF \in \mathbb{R}$, we say that a set of points $\{x_i\}_{i=1}^n$ is $\gamma-$shattered by $\cF$ if there exists $\{s_i\}_{i=1}^n \in \mathbb{R}$ such that for all binary vectors $\{\sigma_i\}_{i=1}^n$,  there is a function $f \in \cF$ satisfying: 
\begin{align*}
    f(x_i) \geq s_i + \gamma & \qquad if \ \sigma_i =1 \\
    f(x_i) \leq s_i - \gamma & \qquad otherwise
\end{align*}{}
\end{thmappdef}{}
The fat-shattering dimension can be thought of as a function from the positive reals to the set of positive integers which maps $\gamma$ to the largest $\gamma-$shattered set or $\infty$. 

We define the empirical proportion overestimated as: 
 \begin{thmappdef}
For $f \in \cF$, $\gamma >0$, a sample $z = \{ x_i, y_i \}_i^n$ drawn from a fixed but unknown distribution $p_t$, known weights $\boldsymbol{w}$, we define the empirical risk when the distribution with respect to $p$: 
 \begin{align*}
  \underline{\epsilon}^{\boldsymbol{w}}_{f}(z, \gamma) & = \sum_i  w(x) \indic \{  \underline{r}_{f}(x,y) < \gamma \}.
\end{align*}
\end{thmappdef}

\section{Proof of theorem~\ref{thm:soft_margin}}

To construct the proof, we will first study the overestimation risk when there are no training set violations (Lemma~\ref{lem:hard_margin}). To extend our results to cases where there are training set violations, we rely on a technique, presented in \cite{shawe2002generalisation} and used in \cite{scholkopf01}, which allows us to ignore small violations in the training data at the cost of a more complex function space. This function space (formally defined in definition \ref{def:aux_func}) is constructed by creating an ``auxiliary function'' that picks specific points to have a non-zero violation. Its complexity depends on the allowable violations.
By augmenting the result from lemma~\ref{lem:hard_margin} with the auxiliary function space, we get theorem~\ref{thm:soft_margin_general}, a general version of theorem~\ref{thm:soft_margin}, which gives a bound on the overestimation risk for general sturdy function spaces. Finally, we give the proof for linear function spaces, which is presented in theorem~\ref{thm:soft_margin} in the main text. 

To build up to lemma~\ref{lem:hard_margin}, we restate the following two previously established results.

\begin{thmapplem}\label{lem:inf_gamma}
Due to \citet{st_williamson_99}: Let $\cF$ be a sturdy function class, then for each $N \in \mathbb{N}^+$ and any fixed sequence $X \in \cX^n$ the infimum $$
\inf \{\gamma : \cN(\gamma, \cF, X) < N \}
$$ 
is attained
\end{thmapplem}

We assume that $f^1_l$, $f^0_l$, $f^0_l$ and $f^0_u$ belong to a sturdy function class, as defined in definition~\ref{def:sturdy}. 

The following lemma due to \citet{cortes} bounds the second moment of the weighted loss. 
\begin{thmapplem}\label{lem:cortes}
Due to \citet{cortes}. For $x \in \cX$, a weighting function $w_t$ on $\cX$, a loss function $\ell$, and some function $f \in \cF$, the second moment of the importance weighted loss can be bounded as follows:
\begin{align*}
    \mathbb{E}_{X \mid T}\left[ w_t^2(X) \ell_f^2(X) \mid T=t \right] \leq d_2(p||p_t)
\end{align*}
\end{thmapplem}

We now study the overestimation error when there are no training set violations, i.e., when $D=0$. A direct analogy can be drawn between the following lemma (lemma~\ref{lem:hard_margin}) and hard margin one-class SVMs studied in \citet{scholkopf01}, whereas theorem~\ref{thm:soft_margin} is analogous to the soft margin case. 

\begin{thmapplem}\label{lem:hard_margin}
Let $\cF$ be the class of linear functions in a kernel defined feature space, $ z = \{ x_i, y_i\}_{i: t_i = t}$, where $x_i, y_i \sim p_t(X, Y)$, and $C_t$ be as defined in~\eqref{eqn:cortes_c}. For $f^t_l \in \cF$, and any $\gamma>0$, let the associated $\underline{D}^{\boldsymbol{w}_t}(z, f_t^1, \gamma ) = 0$. With a probability $ 1-\delta$ over the draw of random samples, we have that: 
\begin{equation}
\begin{aligned}
  \underline{R}_{f^l_t}(\gamma) & \leq \frac{4 C_t (k_t + \log\frac{1}{\delta} )}{3n_t} + \sqrt{\frac{8 d_2(p||p_t) ( k_t + \log\frac{1}{\delta})}{n_t}}.
 \end{aligned}
\end{equation}
where, for $t\in \{0,1\}$,
\begin{align*}
   k_t  & = \bigg\lceil \log \cN(\gamma, \cF, 2n_t) \bigg\rceil~.
\end{align*}
\end{thmapplem}

\begin{thmproof}
For a given $f^1_l \in \cF$:
\begin{align*}
     P \Big(  \underline{R}_{f^1_l}(\gamma) - \underline{\epsilon}_{f^1_l}^{\boldsymbol{w}}(z, \gamma) > \varepsilon \Big)  & = P \Big(  \underline{R}_{f^1_l}(\gamma)  > \varepsilon \Big) \\
    & \leq  2 P \Big(  \underline{\epsilon}_{f^1_l}^{\boldsymbol{w}'}(z', \gamma)  > \frac{\varepsilon}{2} \Big), 
\end{align*}
where the equality follows from the fact that the empirical error on the estimation data will always be 0 by definition of $\gamma$. And the inequality follows from applying the double (ghost) sample trick. Suppose that such an $f^1_l$ exists. Pick a fixed $k$ such that
\begin{align*}
    \gamma_k = \inf \{\gamma: \cN(\gamma, \cF, 2n_1) \leq 2^k \} \leq \gamma~.
\end{align*}
By Lemma \ref{lem:inf_gamma}, and assumption of sturdiness, we have that this $\gamma_k$ exists. Consider the $\gamma_k$-covering, $U$.
There exists another $f_{\bdt} \in U$ such that the distance between $f^1_l$ and $f_{\bdt}$ is $\leq \gamma_k \leq \gamma$, meaning $f_{\bdt}$ satisfies:
\begin{align*}
    P \Big(  \underline{\epsilon}_{f^1_l}^{\boldsymbol{w}'}(z', \gamma) > \frac{\varepsilon}{2} \Big) = P \Big(  \underline{\epsilon}_{f_{\bdt}}^{\boldsymbol{w}'}(z', 0) > \frac{\varepsilon}{2} \Big)
\end{align*}
This limits the complexity of the function class from infinite to having a covering number $=\cC_\cF^\gamma$.  Swapping samples between the estimation and the ghost sample, this will create a random variable $S' = \frac{1}{M}( \underline{\epsilon}_{f_{\bdt}}^{\boldsymbol{w}'_1}(z'_1, 0) + \ldots + \underline{\epsilon}_{f_{\bdt}}^{\boldsymbol{w}'_m}(z'_m, 0), + \ldots + \underline{\epsilon}_{f_{\bdt}}^{\boldsymbol{w}'_M}(z'_M, 0))$ for $M=2^{n_1}$, where the subscripts of $\boldsymbol{w}'$ and $z'$ denote the sample index. Note that $\E_{x\sim p_t}[S'] = \underline{R}_{f_{\bdt}}(0)$ and let $S$ denote $S' -\E_{x\sim p_t}[S']$, with $\E_{x\sim p_t}[S] =0$. Let $\sigma^2(S) = \E[S^2] = \E[(S' -\E_{x\sim p_t}[S'])^2]$. By Lemma~\ref{lem:cortes}, we have that $\sigma^2(S') \leq d_2(p || p_1) - \underline{R}_{f_{\bdt}}(0)^2$. 
By Bernstein's inequality: 
\begin{align*}
    P \Big( \underline{R}_{f_{\bdt}}(0) -  \underline{\epsilon}_{f_{\bdt}}^{w'}(z', 0)  > \frac{\varepsilon}{2} \Big) \leq \exp \Big(\frac{-3n_1 \varepsilon^2}{24\sigma^2(S) + 4C_1 \varepsilon}\Big), 
\end{align*}
and a union bound over the function space:
\begin{align*}
 & P \Big( \underline{R}_{f_{\bdt}}(0) -  \underline{\epsilon}_{f_{\bdt}}^{w'}(z',0)  > \frac{\varepsilon}{2} \Big) \leq\\
 & \cN(\gamma, \cF, 2n_1)  \exp \Big(\frac{-3n_1 \varepsilon^2}{24\sigma^2(S) + 4C_1 \varepsilon}\Big)
\end{align*}
Putting it all together: 
\begin{align*}
        & P \Big(  \underline{R}_{f^1_l}(\gamma) - \underline{\epsilon}_{f^1_l}^{w}(z, \gamma) > \varepsilon \Big) \\
        & \leq 2 P \Big( \underline{R}_{f_{\bdt}}(0) -  \underline{\epsilon}_{f_{\bdt}}^{w'}(z', 0)  > \frac{\varepsilon}{2} \Big)\\
        & \leq 2 \cN(\gamma, \cF, 2n_1)  \exp \Big(\frac{-3n_1 \varepsilon^2}{24\sigma^2(S) + 4C_1 \varepsilon}\Big)
\end{align*}
Setting $\delta(\epsilon)$ to match the upper bound, inverting w.r.t. $\epsilon$ and removing the (negative) term $\underline{R}_{f_{\bdt}}(0)^2$ from the right-hand side, we get that stated bound with probability $1-\delta$.
\qed
\end{thmproof}

Next, we define the auxiliary function space, which will allow us to study non-zero training set violations. 
\begin{thmappdef}\label{def:aux_func}
[Restated from \cite{scholkopf01}, definition 13] Let $L(\cX)$ be the set of real valued, non-negative functions $f$ on $\cX$ with support supp$(f)$ countable, that is the functions in in $L(\cX)$ are non-zero for at moust countably many points. We define the inner product of two functions $f, g \in L(\cX)$ by:
\begin{align*}
    f \cdot g \sum_{x \in \text{supp}(f)} f(x) g(x).
\end{align*}{}
The 1-norm on $L(\cX)$ is defined by $||f||_1 = \sum_{x \in \text{supp}(f)} f(x)$. Let $L^D(\cX):= \{ f \in L(\cX): ||f||_1 \leq D \}$. Define a transformation, or embedding of $\cX$ into the product space $\cX \times L(\cX)$ as follows: 
\begin{align*}
    \varpi & : \cX \rightarrow \cX \times L(\cX)\\
    \varpi & : x \rightarrow (x, \Delta_x), 
\end{align*}{}
where 
\begin{align*}
    \Delta_x = 
    \begin{cases}
       1,&  y = x,\\
       0,& \text{otherwise}
    \end{cases} \\
\end{align*}{}
For a function $f \in \cF$ a set of training examples $z$ of size $n$, define the function $g_f \in L(\cX)$ 
\begin{align*}
     g_f (\mathbf{y}):= \sum_{x,y \in z}  w_1(x) \min\{0, \gamma - \underline{r}_{f^1_l}(x,y) \} \Delta_x(\mathbf{y}), 
\end{align*}{}
where $\mathbf{y} = \{y_i\}_{i=1}^n$
\end{thmappdef}{}

We can now state the risk of overestimation for general sturdy functions. 
\begin{thmappthm}\label{thm:soft_margin_general}
Let $\cF$ be any sturdy function class defined over input space $\cX$, $ z = \{ x_i, y_i\}_{i: t_i = t}$, where $x_i, y_i \sim p_t(X, Y)$, and $C_t$ be as defined in~\eqref{eqn:cortes_c}. For $f^t_l \in \cF$, and any $\gamma>0$, let the associated $\underline{D}^{\boldsymbol{w}_t}(z, f_t^1, \gamma ) = D >0$. With a probability $ 1-\delta$ over the draw of random samples, we have that: 
\begin{equation}
\begin{aligned}
  \underline{R}_{f^l_t}(\gamma) & \leq \frac{4 C_t (k_t + \log\frac{1}{\delta} )}{3n_t} + \sqrt{\frac{8 d_2(p||p_t) ( k_t + \log\frac{1}{\delta})}{n_t}}.
 \end{aligned}
\end{equation}
where, for $t\in \{0,1\}$,
\begin{align*}
   k_t  & = \bigg\lceil \log \cN(\sfrac{\gamma}{2}, \cF, 2n_t)  + \log \cN(\sfrac{\gamma}{2}, L^D(\cX), 2n_t)  \bigg\rceil~.
\end{align*}
\end{thmappthm}{}

\begin{thmproofsketch}
The proof extends lemma~\ref{lem:hard_margin}, replacing the function class $\cF$ with the function class of the augmented space, that is $\cF + L(\cX) := \{ f + g : f \in \cF, g \in L(\cX) \}$. The details of the proof are identical to theorem 14 in \citet{scholkopf01}, and are hence omitted. 
\end{thmproofsketch}{}

The following lemma, restated from \citet{shawe2002generalisation} gives a bound on the auxiliary function complexity for linear functions (defined in kernel spaces). 
\begin{thmapplem}\label{lem:aux_linear}
Due to \citet{shawe2002generalisation}.
For $D>0$, all $\gamma >0$:
\begin{align*}
    & \log \cN(\gamma, L^D(\cX ), n)  \\
    &  \leq \bigg\lfloor \frac{D}{2\gamma} \bigg\rfloor   \log \bigg( \frac{\exp(n + \sfrac{D}{2\gamma} - 1) }{\sfrac{D}{2\gamma}} \bigg)
\end{align*}{}
\end{thmapplem}{}

Finally, by replacing the auxiliary function term from theorem~\ref{thm:soft_margin_general} (that is  $\log \cN(\sfrac{\gamma}{2}, L^D(\cX), 2n_t)$) with its bound for linear functions acquired from lemma~\ref{lem:aux_linear} (that is $\log \frac{\exp(n_t + \sfrac{D}{\gamma} -1)}{\sfrac{D}{\gamma}}$), we get the proof for theorem~\ref{thm:soft_margin}.

\section{Risk of overestimation of ITE}
The risk of overestimation for the ITE can be stated as a simple extension of theorem~\ref{thm:soft_margin}. We define the ITE as $\tau(x) = Y(x, 1) - Y(x, 0)$, where $Y(x, t)$ is the potential outcome under treatment $T=t$, for patient with characteristics $X = x$. We use $\tilde{\tau}_l(x)$ to denote $f^1_l(x) - f^0_u(x)$, where $f^1_l, f^0_u$ are some estimates of the lower bound for the outcome under treatment and the upper bound of the outcome under non-treatment respectively. In addition, we define: 
$$
    \overline{r}_f(x, y) = f(x) -y, 
$$
and for $z_t = \{ x_i, y_i\}_{i: t_i = t}$, define
\begin{align*}
    \overline{D}^{\boldsymbol{w}_t}(z, f^t_u, \gamma) = \sum_{x,y \in z}  w_t(x) \min\{0, \gamma - \overline{r}_{f^t_u}(x,y) \} 
\end{align*}

\begin{thmappcol}\label{col:ite_fcr}
Let $\cF$ be the class of linear functions in a kernel defined feature space, $ z_t = \{ x_i, y_i\}_{i: t_i = t}$, where $x_i, y_i \sim p_t(X, Y)$, and $C_t$ be as defined in expression~\eqref{eqn:cortes_c}. For $f^1_l, f^0_u \in \cF$, and any $\gamma>0$, let the associated $\underline{D}^{\boldsymbol{w}_1}(z_1, f^1_l, \gamma ) = D_1 >0$, and $\overline{D}^{\boldsymbol{w}_0}(z_0, f^0_u, \gamma ) = D_0 >0$ Define  $\tilde{\tau}_l := f^1_l- f^0_u$. With probability $1-\delta$ over random samples, we have that:
\begin{equation}
\begin{aligned}
  \underline{R}_{\hat{\tau}_l}(\gamma) & \leq \sum_t \frac{4 C_t (k_t + \log\frac{1}{\delta} )}{3n_t} \\ & + \sqrt{\frac{8 d_2(p||p_t) ( k_t + \log\frac{1}{\delta})}{n_t}}.
 \end{aligned}
\end{equation}
where, for $t\in \{0,1\}$,
\begin{align*}
   k_t  & = \bigg\lceil \log \cN(\sfrac{\gamma}{2}, \cF, 2n_t)  + \log \cN(\sfrac{\gamma}{2}, L^{D_t}(\cX), 2n_t)  \bigg\rceil~.
\end{align*}
\end{thmappcol}

\begin{thmproof}
Consider the event:
$$
E = \big\{ x : \tau(x) < \tilde{\tau}_l(x) - 2\gamma \big\}
$$
where $x \sim p$. Note that event $E$ implies that one of the following two events must hold:
$$
E_1 = \big\{ (x, y) : \underline{r}_{f^1_l}(x,y) < \gamma \big\}
$$
for $t=1$. 
$$
E_0 = \big\{ (x, y_0) : \overline{r}_{f^0_u}(x,y)  < \gamma  \big\}
$$
for $t=0$. 

Note that $p(E_1) = \underline{R}_{f^1_l}(\gamma)$. So, theorem~\ref{thm:soft_margin_general} implies that
\begin{align*}
   & p(E_1)\leq \frac{4 C_t (k_t + \log\frac{1}{\delta} )}{3n_t} + \sqrt{\frac{8 d_2(p||p_t) ( k_t + \log\frac{1}{\delta})}{n_t}}
\end{align*}
for $k_t$ as defined in theorem~\ref{thm:soft_margin_general}.
Similarly $p(E_0) = \overline{R}(f^0_u)$, and by a similar construction can obtain the bound on $p(E_0)$. 
Using a union bound we have that
\begin{align*}
    p(E) & = p(E_1 \cup E_0) = p(E_1) + p(E_0) - p(E_1 \cap E_0)\\
    & \leq p(E_1) + p(E_0), 
\end{align*}
which completes the proof.
\qed
\end{thmproof}

\section{Proof of Theorem~\ref{thm:soft_margin_fat}}
To build up to the proof of theorem~\ref{thm:soft_margin_fat}, we first seek a bound on the fat-shattering dimension of functions defined in definition~\ref{def:lb_class}. This bound is constructed in a similar spirit to theorem 1.6 in \cite{bartlett1999generalization}. Specifically, to get a bound on the fat-shattering dimension, we rely on the lemmas~\ref{lem:fat_lower} and ~\ref{lem:fat_upper}. The former shows that the sum of any shattered set is far from the remainder of that set, the latter shows that the same sums cannot be too far apart. 

\begin{thmapplem}\label{lem:fat_lower}
Let $\cF_u, \cF_l, A, B$ be as defined in definition~\ref{def:lb_class}. Let  $ I = \{ x_i\}_{i=1}^n$, where $x_i \sim p(X, Y)$.For a fixed $\gamma >0$, if $I$ is $\gamma-$shattered by $\cF_l$ then every subset $I' \in I$ satisfies: 
\begin{align*}
    \min_{q \in \{p, 2\}}\bigg\| \sum_{i \in I'}x_i - \sum_{i \in I\setminus I'}x_i \bigg\|_q \geq \frac{2n\gamma}{A + B}
\end{align*}{}
\end{thmapplem}{}

\begin{thmproof}
If $I$ is $\gamma$ shattered by $\cF_l$, denote the corresponding ``witness'' vector by $\{s_i\}_{i=1}^n$, then for all $\boldsymbol{\sigma}=\{\sigma_1 \hdots \sigma_i \hdots \sigma_n\}$ there is an $f$ with $\|f_l\| \leq A$ such that $\sigma_i\cdot(\theta^\top x_i - s_i) \geq \gamma$ for $i =1 \hdots n$.
Suppose that:
\begin{align}\label{exp:sumgeq}
    \sum_{i \in I'}s_i \geq \sum_{i \in I\setminus'I}s_i 
\end{align}{}
Then fix $\sigma_i = 1$ if $i \in I'$. In that case we have that 
\begin{align}
    \label{exp:theta_id} & \langle f_l, x_i \rangle \geq s_i + \gamma & \forall i \in I' \\
    \label{exp:theta_ind} & \langle f_l, x_i \rangle  < s_i - \gamma & \forall i \in I\setminus I'.
\end{align}{}

Pick $f_u \in \cF_u$ such that $|| f_u - f_l ||_p = B' \leq B$, and:
\begin{align}
    \label{exp:diff_id} & \langle f_u - f_l, x_i \rangle \geq s_i + \gamma & \forall i \in I' \\
    \label{exp:diff_ind} & \langle f_u - f_l, x_i \rangle  < s_i - \gamma & \forall i \in I\setminus I'. 
\end{align}{}
Showing that such a function exists is trivial: simply take $f_u := f_l$. For that we have $||f_u - f_l|| = 0 \leq B$, which means that the function does exist in $\cF_u$. 

From expression~\ref{exp:theta_id}, we have that: 
\begin{align*}
    \big\langle f_l , \sum_{i \in I'} x_i \big\rangle= 
    \sum_{i \in I'} \langle f_l,  x_i \rangle
    \geq \sum_{i \in I'}s_i + Card(I')\gamma, 
\end{align*}{}
where $Card(.)$ denotes the cardinality. Similarly for $I\setminus I'$, we have that 

\begin{align*}
    \big\langle f_l , \sum_{i \in I\setminus I'} x_i \big\rangle
     &< \sum_{i \in I\setminus I'} s_i + Card(I\setminus I')\gamma
\end{align*}{}

Combining the expressions for $I'$ and $I \setminus I'$, and from expression~\ref{exp:sumgeq}:
\begin{align}\label{exp:theta_sums}
    \big\langle f_l,   \sum_{i \in I'} x_i - \sum_{i \in I\setminus I'} x_i \big\rangle \geq n\gamma.
\end{align}{}

We now construct the same arguments for the distance. Let $f_d := f_u - f_l$. From expression~\ref{exp:diff_id}, we have that: 
\begin{align*}
    \big\langle f_d , \sum_{i \in I'} x_i \big\rangle= 
    \sum_{i \in I'} \langle f_d,  x_i \rangle
    \geq \sum_{i \in I'}s_i + Card(I')\gamma, 
\end{align*}{}

and from expression~\ref{exp:diff_ind}:
\begin{align*}
    \big\langle f_d , \sum_{i \in I\setminus I'} x_i \big\rangle
     &< \sum_{i \in I\setminus I'} s_i + Card(I\setminus I')\gamma
\end{align*}{}

Combining the two, and from expression~\ref{exp:sumgeq}:
\begin{align}\label{exp:diff_sums}
    \big\langle f_d,   \sum_{i \in I'} x_i - \sum_{i \in I\setminus I'} x_i \big\rangle \geq n\gamma.
\end{align}{}

Putting expressions~\ref{exp:theta_sums} and~\ref{exp:diff_sums} together,
\begin{align}\label{exp:both_sums}
    &\big\langle f_l,   \sum_{i \in I'} x_i - \sum_{i \in I\setminus I'} x_i \big\rangle \\
    & + \big\langle f_d,   \sum_{i \in I'} x_i - \sum_{i \in I\setminus I'} x_i \big\rangle
     \geq 2n\gamma.
\end{align}{}

Note that by Cauchy-Schwartz, 
\begin{align*}
   \big\langle f_l,   \sum_{i \in I'} x_i - \sum_{i \in I\setminus I'} x_i \big\rangle  & \leq \| f_l \| \bigg\|  \sum_{i \in I'} x_i - \sum_{i \in I\setminus I'} x_i \bigg\| \\
    & \leq A \bigg\|  \sum_{i \in I'} x_i - \sum_{i \in I\setminus I'} x_i \bigg\| \\
   &   \leq A  \min_{q \in \{p, 2\}}\bigg\| \sum_{i \in I'}x_i - \sum_{i \in I\setminus I'}x_i \bigg\|_q.  
\end{align*}{}

and, 
\begin{align*}
   \big\langle f_d,   \sum_{i \in I'} x_i - \sum_{i \in I\setminus I'} x_i \big\rangle & \leq \| f_d \|_p \bigg\|  \sum_{i \in I'} x_i - \sum_{i \in I\setminus I'} x_i \bigg\|_p \\
     & \leq B' \bigg\|  \sum_{i \in I'} x_i - \sum_{i \in I\setminus I'} x_i \bigg\|_p \\
     &  \leq B \bigg\|  \sum_{i \in I'} x_i - \sum_{i \in I\setminus I'} x_i \bigg\|_p \\
   &  \leq B \min_{q \in \{p, 2\}}\bigg\| \sum_{i \in I'}x_i - \sum_{i \in I\setminus I'}x_i \bigg\|_q.
\end{align*}{}

For expression~\ref{exp:both_sums} to hold: 
\begin{align*}
   & A  \min_{q \in \{p, 2\}}\bigg\| \sum_{i \in I'}x_i - \sum_{i \in I\setminus I'}x_i \bigg\|_q   \\
   & +  B  \min_{q \in \{p, 2\}}\bigg\| \sum_{i \in I'}x_i - \sum_{i \in I\setminus I'}x_i \bigg\|_q  \geq 2n\gamma \\
    & (A + B)\min_{q \in \{p, 2\}}\bigg\| \sum_{i \in I'}x_i - \sum_{i \in I\setminus I'}x_i \bigg\|_q  \geq 2n\gamma \\
    & \min_{q \in \{p, 2\}}\bigg\| \sum_{i \in I'}x_i - \sum_{i \in I\setminus I'}x_i \bigg\|_q  \geq \frac{2n\gamma}{(A+ B)}, 
\end{align*}{}
which completes the proof. 

\end{thmproof}{}

\begin{thmapplem}\label{lem:fat_upper}
Let $\cF_u, \cF_l, r$ be as defined in definition~\ref{def:lb_class}. Let  $ I = \{ x_i\}_{i=1}^n$, where $x_i \sim p(X, Y)$.For a fixed $\gamma >0$, if $I$ is $\gamma-$shattered by $\cF_l$ then every subset $I' \in I$ satisfies: 
\begin{align*}
   \bigg\| \sum_{i \in I'}x_i - \sum_{i \in I\setminus I'}x_i \bigg\| \leq \sqrt{n} r
\end{align*}{}
\end{thmapplem}{}

The proof is identical to Lemma 1.3 in \cite{bartlett1999generalization}, and is hence omitted.

\begin{thmapplem}\label{lem:bd_fat}
    Let $\cF_u, \cF_l, A, B, r$ be as defined in definition~\ref{def:lb_class}. For a fixed $\gamma >0$, the $\gamma-$fat shattering dimension of $\cF_l$ can be bounded as follows: 
    \begin{align*}
        \text{fat}(\gamma, \cF_l) \leq \bigg( \frac{r \cdot (A + B)}{2 \gamma} \bigg)^2
    \end{align*}{}
\end{thmapplem}{}

Combining the results from Lemmas~\ref{lem:fat_upper} and~\ref{lem:fat_lower}, we get that: 
\begin{align*}
    \frac{2n\gamma}{A + B} \leq  \min_{q \in \{p, 2\}}\bigg\| \sum_{i \in I'}x_i - \sum_{i \in I\setminus I'}x_i \bigg\|_q \\
    \leq  \bigg\| \sum_{i \in I'}x_i - \sum_{i \in I\setminus I'}x_i \bigg\| \leq \sqrt{n} r,
\end{align*}{}
which gives us that:
\begin{align*}
     \sqrt{n} \leq \frac{r(A + B)}{2\gamma}, 
\end{align*}{}
which completes the proof.

\begin{thmappthm}\label{lem:soft_margin_fat_supp}
Let $\cF^t_l$, $\cF^t_u$, $A$, $B$, and $r$ be as defined in definition~\ref{def:lb_class}, $z$, and $D$ as defined in theorem~\ref{thm:soft_margin},and $C_t$ be as defined in expression~\eqref{eqn:cortes_c}. For $f^t_l \in \cF^t_l$, $f^t_u \in \cF^t_u$ and any $\gamma>0$, with a probability $1-\delta$ over the draw of random samples, we have that:
\begin{equation}
\begin{aligned}
  \underline{R}_{f^l_t}(\gamma) & \leq \frac{4 C_t (k_t + \log\frac{1}{\delta} )}{3n_t} + \sqrt{\frac{8 d_2(p||p_t) ( k_t + \log\frac{1}{\delta})}{n_t}}.
 \end{aligned}
\end{equation}
where, for $t\in \{0,1\}$,
\begin{align*}
   k_t   & = \bigg\lceil \bigg(\frac{2r(A+B)}{\gamma}\bigg)^2  \log\bigg( \frac{8n_t(b-a)^2}{\gamma^2} \bigg) \\ & \log \bigg(\frac{4en_t(b-a)\gamma}{r^2(A+B)^2} \bigg)
   + \frac{D}{\gamma} \log \frac{e(n_t + \sfrac{D}{\gamma} -1)}{\sfrac{D}{\gamma}} \bigg\rceil~.
\end{align*}
\end{thmappthm}{}
Using Corollary 3.8 \cite{shawe1998structural}, we can $\log \cN(\sfrac{\gamma}{2}, \cF, 2n_t)$ by its fat shattering dimension. Combining the results from lemma~\ref{lem:bd_fat} and theorem~\ref{thm:soft_margin}, we get the final result.


\section{Equivalence to quantile regression}

Consider the following problem 
\begin{equation}\label{eq:separated}
\begin{aligned}
& \underset{f_u, f_l}{\text{minimize}}
& & \ell^{(1)}_{\tilde{w}}(f_u(x_i), f_l(x_i)) \\
& \text{subject to}
& & \sum_{i:t_i=t} \tilde{w}_{t_i} \max[y_i - f_u(x_i), 0] \leq \beta \\
& & & \sum_{i:t_i=t} \tilde{w}_{t_i} \max[f_l(x_i) - y_i, 0] \leq \beta \\
& & & f_u(x_i) \geq f_l(x_i), \ \ \forall i: t_i = t
\end{aligned}
\end{equation}

\begin{thmappthm}
Assume that \eqref{eq:separated} is strictly convex and has a strictly feasible solution. Then, for any fixed quantile $t \in (0.5, 1)$, there is a parameter $\beta \geq 0$ such that the minimizer of \eqref{eq:separated} with weighted absolute loss and the minimizer of the werighted quantile loss, for quantiles $(t, 1-t)$ with non-crossing constraints, are equal and have false coverage rate $1-q$.
\end{thmappthm}
\begin{proof}
Problem \eqref{eq:separated} with absolute loss $\ell(y, y') = |y - y'|$ can be stated as 
\begin{equation*}
\begin{aligned}
& \underset{f_u, f_l}{\text{minimize}}
& & \sum_{i: t_i =t} \tilde{w}_{t_i}| f_u(x_i)- f_l(x_i)| \\
& \text{subject to}
& & \sum_{i:t_i=t} \tilde{w}_{t_i} \max[y_i - f_u(x_i), 0] \leq \beta \\
& & & \sum_{i:t_i=t} \tilde{w}_{t_i} \max[f_l(x_i) - y_i, 0] \leq \beta \\
& & & f_u(x_i) \geq f_l(x_i), \ \ \forall i: t_i = t
\end{aligned}
\end{equation*}

Let $Q_\beta(f_u,f_l) = \tilde{w}_{t_i}| f_u(x_i)- f_l(x_i)|$ denote the objective and $F$ the feasibility region.
Introducing Lagrange multipliers for the first two constraints, we obtain the regularized objective
\begin{align*}
L(f_u, f_l, \lambda_u, \lambda_l) & = \sum_{i: t_i =t} \tilde{w}_{t_i}|f_u(x_i) - f_l(x_i)| \\
& + \frac{\lambda_u}{n}\sum_{i=1}^n \max(y_i - f_u(x_i), 0) - \beta \\
& + \frac{\lambda_l}{n}\sum_{i=1}^n \max(f_l(x_i) - y_i, 0) - \beta
\end{align*}
and by convexity and strict feasibility, strong duality holds through Slater's condition,
$$
\min_{u, l \in F} Q_\beta(u, l) = \max_{\lambda_u, \lambda_l \geq 0} \min_{u \geq l} L(u, l, \lambda_u, \lambda_l)~.
$$

By strict convexity, for each $\beta \geq 0$, the minimizers $u^*, l^*$ on either side are equal for the maximizers $\lambda_u^*, \lambda_l^*$. Now, consider the  following objective, equivalent in minima to $\tilde{L}(f_u, f_l, \lambda_u, \lambda_l)$, 
\begin{align*}
\tilde{L}(f_u, f_l, \lambda_u, \lambda_l)& := \sum_{i:t_i=t} \tilde{w}_{t_i} |f_u(x_i) - f_l(x_i)| \\
& + \lambda_u \sum_{i:t_i=t} \tilde{w}_{t_i} \max(y_i - f_u(x_i), 0) \\
& + \lambda_l \sum_{i:t_i=t} \tilde{w}_{t_i} \max(f_l(x_i) - y_i, 0)
\end{align*}
We can separate $\tilde{L}$ into terms for which $y_i \geq f_u(x_i)$ and $y_i \geq f_l(x_i)$ respectively, adding and subtracting $\sum_i y_i$
\begin{align*}
&  \tilde{L}(f_u, f_l, \lambda_u, \lambda_l)  \\
& = (\lambda_u - 1) \sum_{y_i \geq u(x_i)} \tilde{w}_{t_i}(y_i - f_u(x_i)) - \sum_{y_i < f_u(x_i)} \tilde{w}_{t_i}(y_i - f_u(x_i)) \\
 & + (1 - \lambda_l)\sum_{y_i \geq f_l(x_i)} \tilde{w}_{t_i}(y_i - f_l(x_i)) - \sum_{y_i < f_l(x_i)} \tilde{w}_{t_i}(y_i - f_l(x_i))
\end{align*}
Now, let $\lambda_u = \lambda_l = 1/(1-q)$ for $q\in (0,1)$, which means $(1-q) \geq 0$. Multiplying by $(1-q)$ leaves us with
\begin{align*}
& \tilde{L}(f_u, f_l,\lambda_u, \lambda_l) \\
& \propto \sum_{y_i \geq f_u(x_i)} q \cdot \tilde{w}_{t_i}(y_i - f_u(x_i)) + \\
& \sum_{y_i < f_u(x_i)} (q-1) \cdot \tilde{w}_{t_i}(y_i - f_u(x_i)) \\
& + \sum_{y_i \geq f_u(x_i)} (1-q) \cdot \tilde{w}_{t_i}(y_i - f_l(x_i)) \\
& +  \sum_{y_i < f_u(x_i)} (-q) \cdot \tilde{w}_{t_i} (y_i - f_l(x_i)) \\
& \propto \sum_{i:t_i=t} \tilde{w}_{t_i} \max[q(y_i - f_u(x_i)), (q-1)(y_i - f_u(x_i)] \\
& +  \sum_{i:t_i=t} \tilde{w}_{t_i} \max[(1-q)(y_i - f_l(x_i)), (-q)(y_i - f_l(x_i)] \\
& = \sum_{i:t_i=t} \rho^{(q)}_{\tilde{w}_{t_i}}(y_i - f_u(x_i)) + \rho^{(1-q)}_{\tilde{w}_{t_i}}(y_i - f_l(x_i))~,
\end{align*}
where $\rho^{(q)}_{\tilde{w}}$ is the weighted quantile loss for quantile $q$. Recalling that our original problem had the constraint $f_u(x_i) \geq f_l(x_i)$, we recover the non-crossing constraint. 
\end{proof}


\section{Cross-validation algorithm}
Define $\Omega$ denote a set of candidate hyperparameters. Suppose we have $M$ possible hyperparameters, cross-validating BP proceeds as follows: 
\begin{algorithm}[H]
  \caption{BP cross-validation for $M$ sets of hyperparameters, and required FCR = $\nu$}
  \label{alg:xval}
\begin{algorithmic}
  \STATE {\bfseries Input:} $\cD = \{x_i, t_i, y_i, w_i\}$, $p$, $\nu$, $\{ \Omega \}^M$
  \STATE {\bfseries Output:} $ \Omega^*$
  \STATE Split $\cD$ into $\cD_{\text{train}}$, $\cD_{\text{validate}}$
  \FOR{$m=1$ {\bfseries to} $M$}
        \STATE Use $\cD_{\text{train}}$ to solve problem~\eqref{eqn:indep_obj_const} or~\eqref{eqn:coupled_obj_const}
        \STATE Estimate $\hat{\nu}^{(m)}$, and $||\widehat{\text{IW}}||_p^{(m)}$ on $\cD_{\text{validate}}$
  \ENDFOR
  \STATE  Define $M' = \{m: \hat{\nu}^{(m)} \leq \nu \}$
  \STATE Set $\Omega^*:= \min_{m \in M'} ||\widehat{\text{IW}}||_p^{(m)}$
\end{algorithmic}
\end{algorithm}

\section{Experiments}

\subsection{Cross-validation details}
For our BP method, we have 5 hyperparameters to pick. These are $\alpha$, the regularization parameter, the kernel bandwidth, $\beta_u$ and $\beta_l$ which are the allowed violations. The last parameter, $\gamma_{BP} >0$, as described in section~\ref{sect:xval}. Note that the kernel bandwidth is only relevant for the experiments done on the ACIC data, but not the IST experiments since a linear kernel is used in the latter. 

For the kernel regression (KR), we first split the training data into 2. On the first half, we do the typical 3-fold cross-validation to pick the model that minimizes the weighted empirical error. This allows us to pick the kernel bandwidth, and a regularization parameter the is multiplied by the L2 norm of the weights. Again, the kernel bandwidth is only relevant for the experiments done on the ACIC data, but not the IST experiments since a linear kernel is used in the latter. The intervals are then estimated in one of two ways. For KR-MI, we use the second part of the training data to estimate the residuals. We follow algorithm 2 in \cite{lei2018distribution} to get the final interval estimates. For KR-$\gamma$, we use the second half of the training data to estimate the FCR, $\hat{\nu}_{\gamma_{\text{KR}}}$, with $\gamma_{KR}$ defined as the ``shifting'' parameter, where $\tilde{f}^{KR}_u(x_i) = \tilde{\mu}_t(x_i) + \gamma_{KR}$ and $\tilde{f}^{KR}_l(x_i) = \tilde{\mu}_t(x_i) - \gamma_{KR}$, for $\tilde{\mu}_t(x_i)$ being the predicted response value. We then pick the smallest $\gamma_{KR}$ that does not violated the required FCR.

For the Gaussian process (GP), we pick the kernel bandwidth, the noise level added to the diagonal of the kernel. For BART models, we use the BartMachine package in R \cite{bartR}. We do 3 fold cross-validation to pick the parameter $k$, which controls the prior probability that $\mathbb{E}(y|x)$ is contained in the interval (ymin,ymax), based on a normal distribution. We set the number of trees to be 200, since that did not seem to affect the results. For the CMGP, we pick the lengthscale of the RBF kernels of the two response surfaces as well as the variance and correlation parameters.

\subsection{Additional IST details}
Figure~\ref{fig:hist} shows the histogram of the ages in the training data for the treated and the control population. Ages$>70$ were downsampled to introduce a confounding effect. 
\begin{figure}[H]
\centering
\includegraphics[width=\columnwidth]{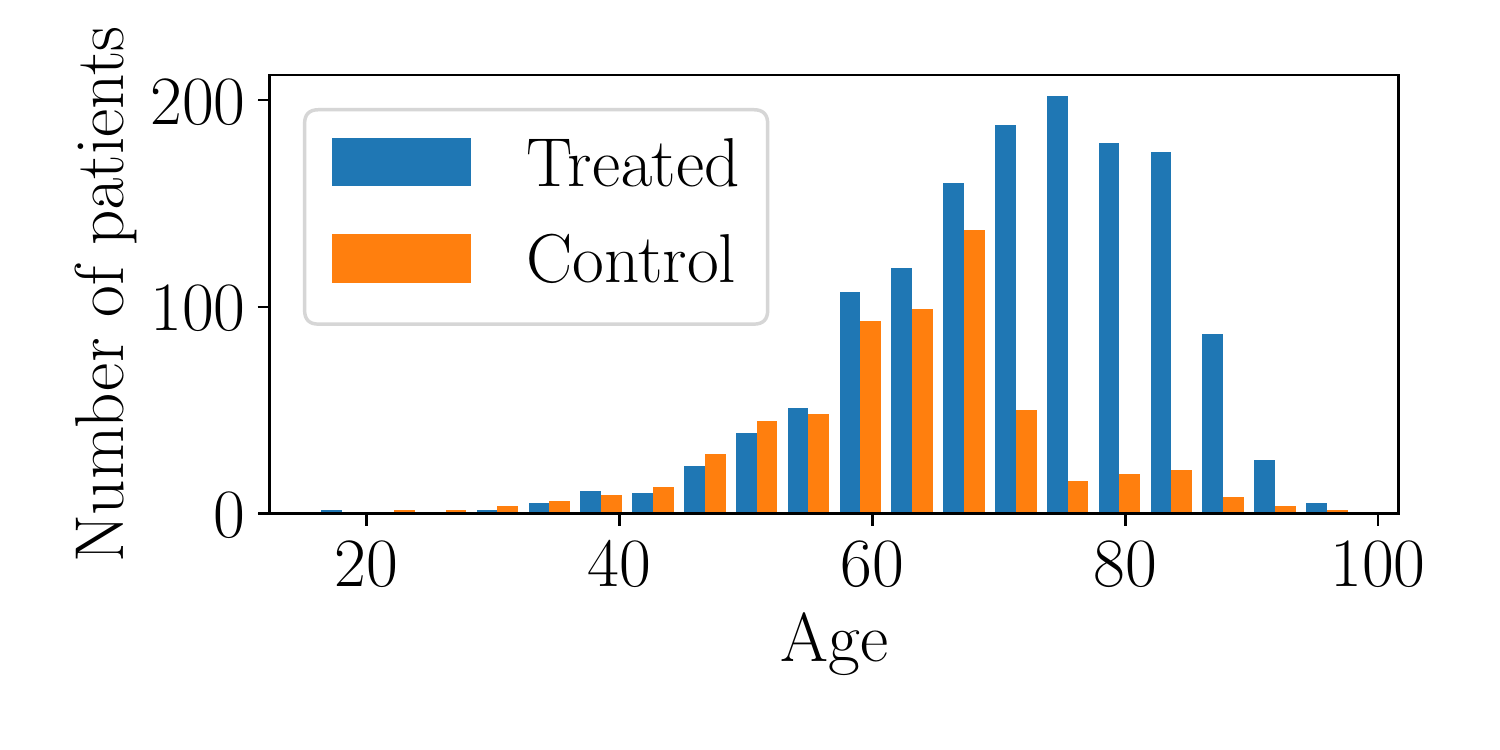}%
\caption{Distribution of data in the IST experiment \label{fig:hist}}
\end{figure}

\subsection{Additional IST results (heteroskedasticity)}\label{sect:ist_hsk}
In this section we analyze the performance of our model when the well-behavedness assumption is violated, specifically when there is heteroskedasticity. We use the IST data, and follow the same train/test splits as is done in the main paper. Here, we focus on the outcome under treatment, $Y(1)$ only. Specifically, we generate the outcome under treatment as $Y(1) = x^2 + \epsilon$, where $x$ is the age rescaled to fall between -2, 2, and $\epsilon_i$ is drawn from a Gaussian distribution with mean 0 and standard deviation = $0.1$ if $x \leq 0$, and from a Gaussian distribution with mean 0 and standard deviation = $0.1 + x$ otherwise. We set the required FCR to be $\leq 0.01$. Since our main aim is to analyze how the different models perform when when heteroskedasticity occurs, we focus only on tightness of bounds as an objective.  

Figure~\ref{tab:ist_hsk} shows the results from averaged over 20 simulations. It shows that of all the models that achieve the required FCR, BP-D-L2 achieves the tightest intervals. Figure~\ref{fig:ist_hsk} shows why: neither BP-D-L2 and QR (equivalent to BP-D-L1) make assumptions about well-behavedness of the residual distribution. They git adaptive intervals, which are tight when the heteroskedastic noise is low, and loose when it is high.

\begin{table}[h]
\centering
\caption{IST heteroskedasticity results. Table shows results averaged over 20 simulations~\ref{fig:ist_hsk}. \label{tab:ist_hsk}}
\begin{tabular}{l|lll}
\toprule
Model &          FCR &       Mean IW &         Max IW \\
\midrule
BP-D-L2         &   0.007 (0.5)  &  5.55 (0.56)  &  10.68 (2.35)  \\
QR/BP-D-L1         &  0.006 (0.31)  &  6.49 (0.96)  &  11.63 (2.37)  \\
KR-$\gamma$     &  0.065 (0.86)  &  3.98 (0.06)  &   3.98 (0.06)  \\
KR-CI           &  0.007 (0.52)  &  6.94 (0.69)  &   6.94 (0.69)  \\
\bottomrule
\end{tabular}
\end{table}

\begin{figure}[h]
    \centering
    \includegraphics[width = \columnwidth]{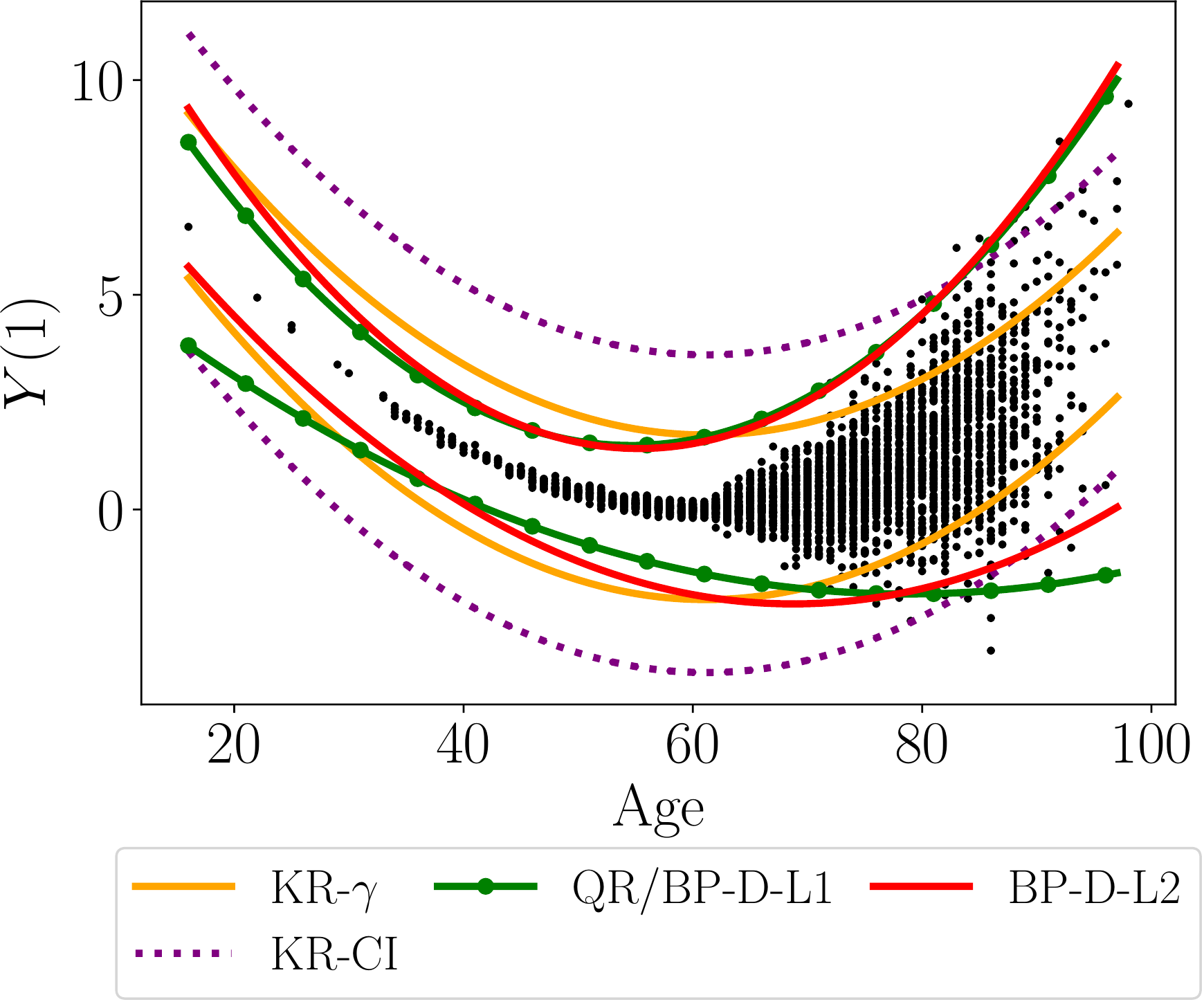}
    \caption{IST heteroskedasticity results. Plot shows results from a single simulation. Black dots show potential outcomes on the test set, lines show fitted values. The plot show that BP-D-L2 and QR (equivalent to BP-D-L1) are the only ones that are able to fit \emph{adaptive} intervals (wider where there is high heteroskedasticity). BP-D-L2 achieves the tightest intervals on average. }
    \label{fig:ist_hsk}
\end{figure}

\subsection{ACIC results including CCI}\label{sect:acic_cci}
Figure~\ref{fig:acic_cci} is similar to figure~\ref{fig:acic} presented in the main paper but includes the performance of CCI models. 

\begin{figure*}
  \begin{minipage}{.5\textwidth}
  \captionsetup[subfigure]{justification=centering}
    \begin{subfigure}{\columnwidth}
        \centering
        \caption{ Comparing tightness of estimated intervals \label{fig:acic0_cci}}
        \includegraphics[width=.9\textwidth]{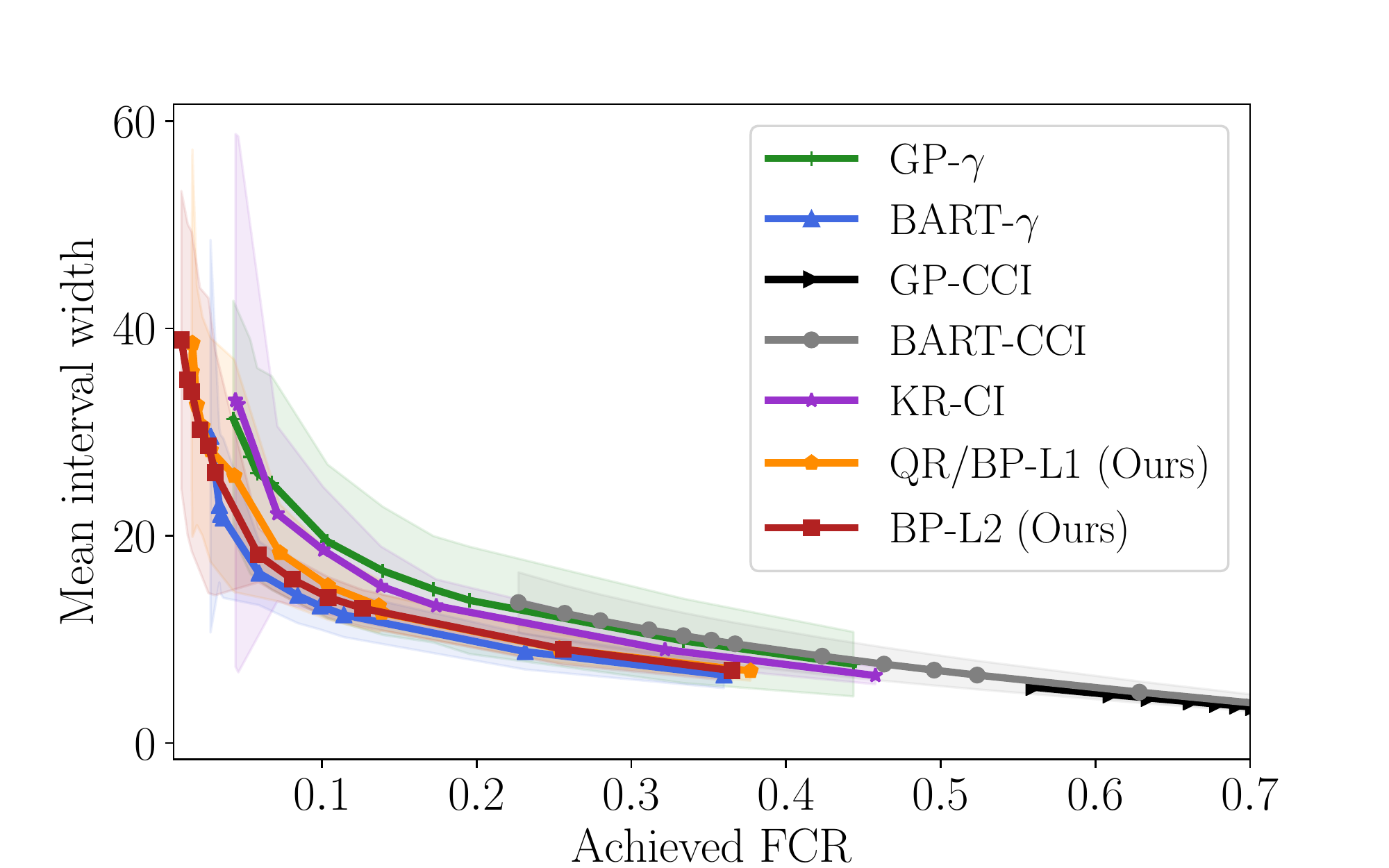}%
    \end{subfigure}
  \end{minipage}
  \begin{minipage}{.50\textwidth}
  \captionsetup[subfigure]{justification=centering}
    \begin{subfigure}{\columnwidth}
        \centering
        \caption{Comparing violation to the required FCR\label{fig:acic1_cci}}
        \includegraphics[width=.9\linewidth]{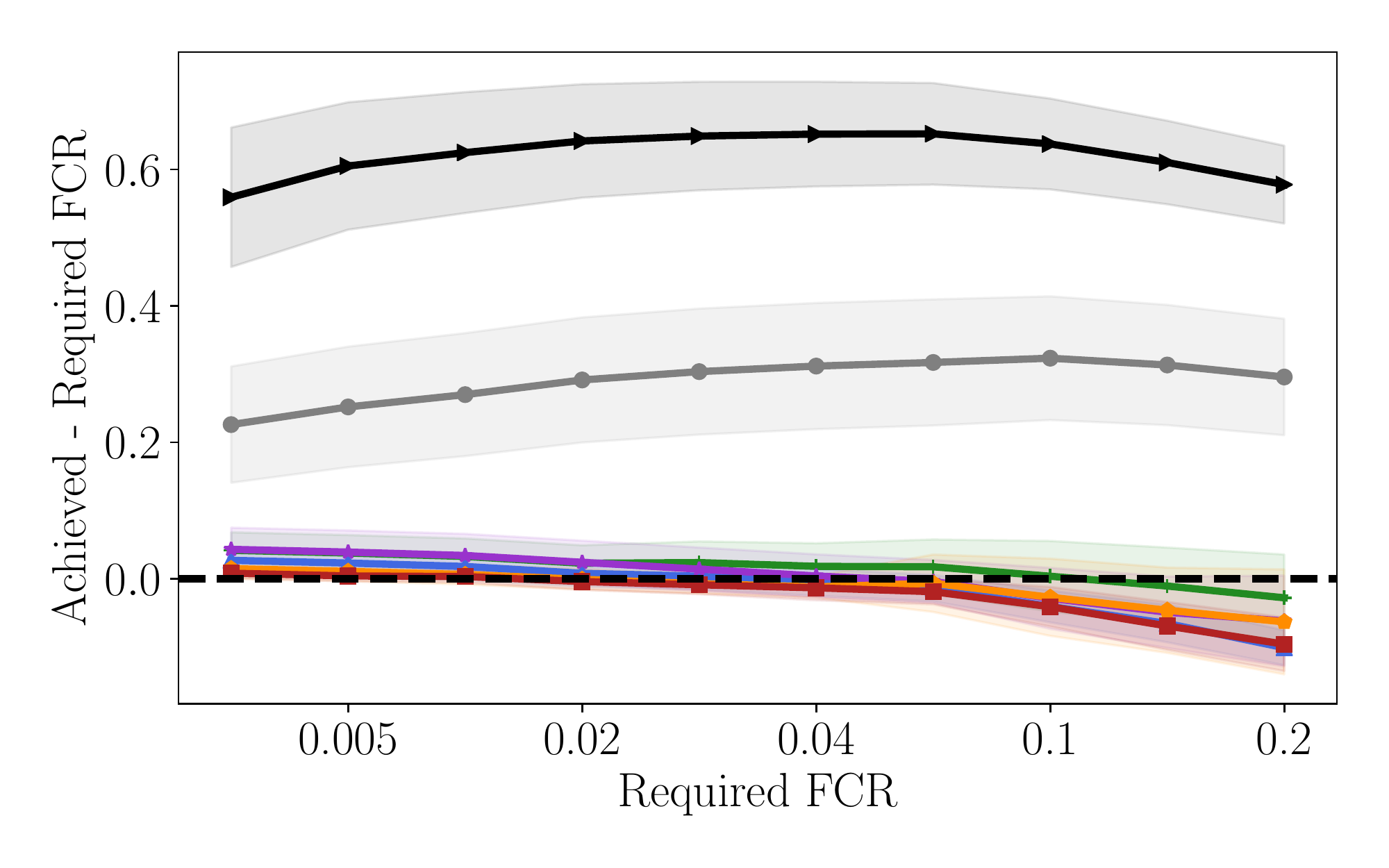}%

    \end{subfigure}
  \end{minipage}

  \caption{ACIC results. Plots show results averaged over 20 simulations. Plot~\ref{fig:acic0_cci} shows the mean interval width for different values of the achieved FCR on a held-out test set. Plot~\ref{fig:acic1_cci} shares the same legend as plot~\ref{fig:acic0_cci}, and shows the violation of the required FCR (= achieved - required) at different values of required FCR. Models above the dotted black line are in violation of the required FCR. The two plots show that BP achieves a mean interval width comparable to that of BART but at a lower violation of the required FCR. BP outperforms all kernel-based methods in terms of mean interval width and violation to the required FCR. CCI methods achieve the worst violations.
\vspace{-1.2em} \label{fig:acic_cci}}
\end{figure*}

\subsection{Additional ACIC results}\label{sect:acic_1k}
We consider a larger sample size than that presented in the main paper. Instead of sampling $n = 200$ for training and validation of the main model, we sample $n = 1000$. In this setting, we are better able to fit the true outcomes since the larger sample size affords us the ability to fit more complex models. Figure~\ref{fig:acic_1k} shows the results. Once again we see that our models outperform all kernel based methods. Here we see that BART-$gamma$ achieves a tighter interval width than our model for the same level of FCR violation. This highlights the strength of tree based models in that they fit highly adaptive ``kernels''. 

\begin{figure*}
  \begin{minipage}{.5\textwidth}
  \captionsetup[subfigure]{justification=centering}
    \begin{subfigure}{\columnwidth}
        \centering
        \caption{ Comparing tightness of estimated intervals \label{fig:acic0_1k}}
        \includegraphics[width=.9\textwidth]{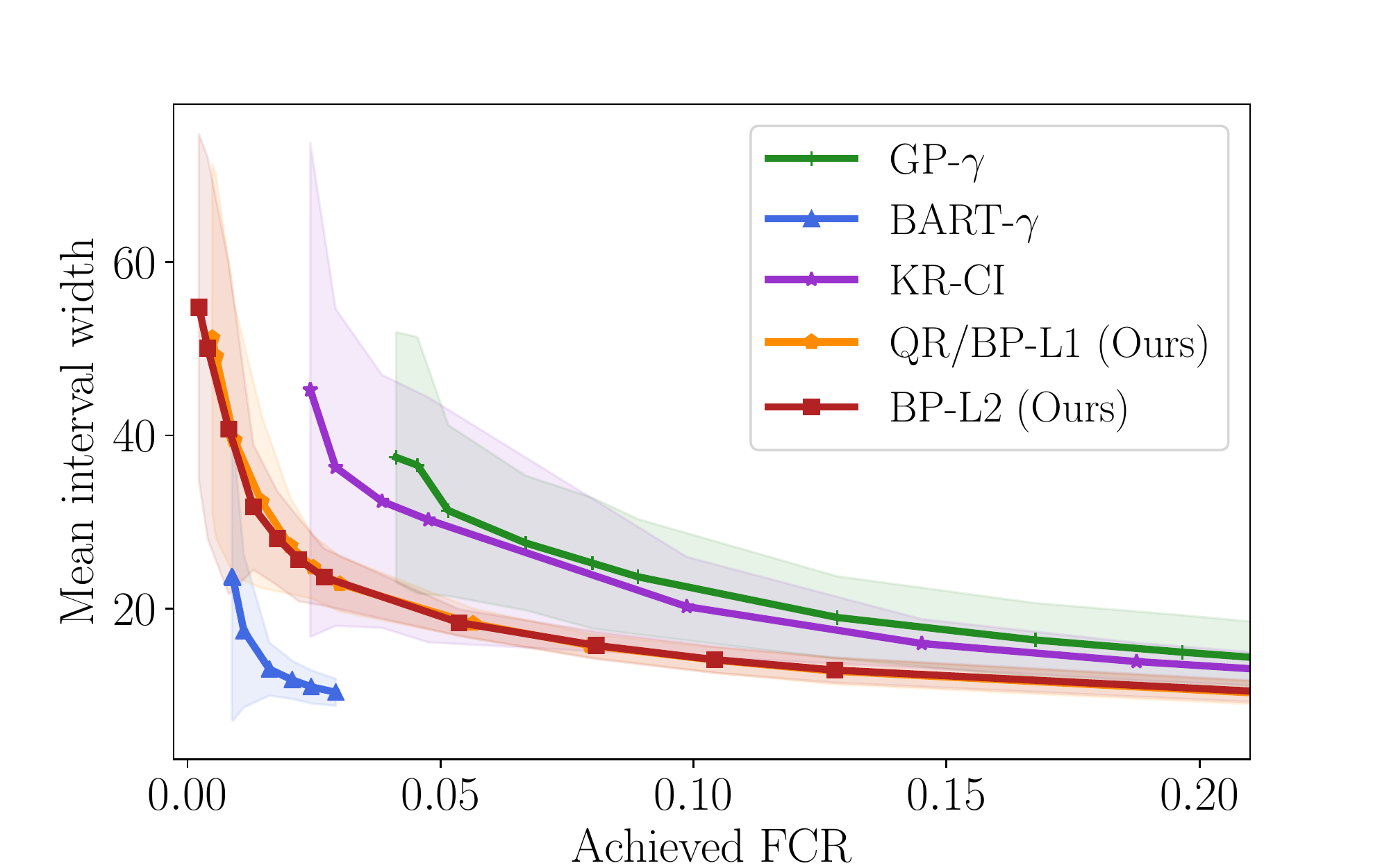}%
    \end{subfigure}
  \end{minipage}
  \begin{minipage}{.50\textwidth}
  \captionsetup[subfigure]{justification=centering}
    \begin{subfigure}{\columnwidth}
        \centering
        \caption{Comparing violation to the required FCR\label{fig:acic1_1k}}
        \includegraphics[width=.9\linewidth]{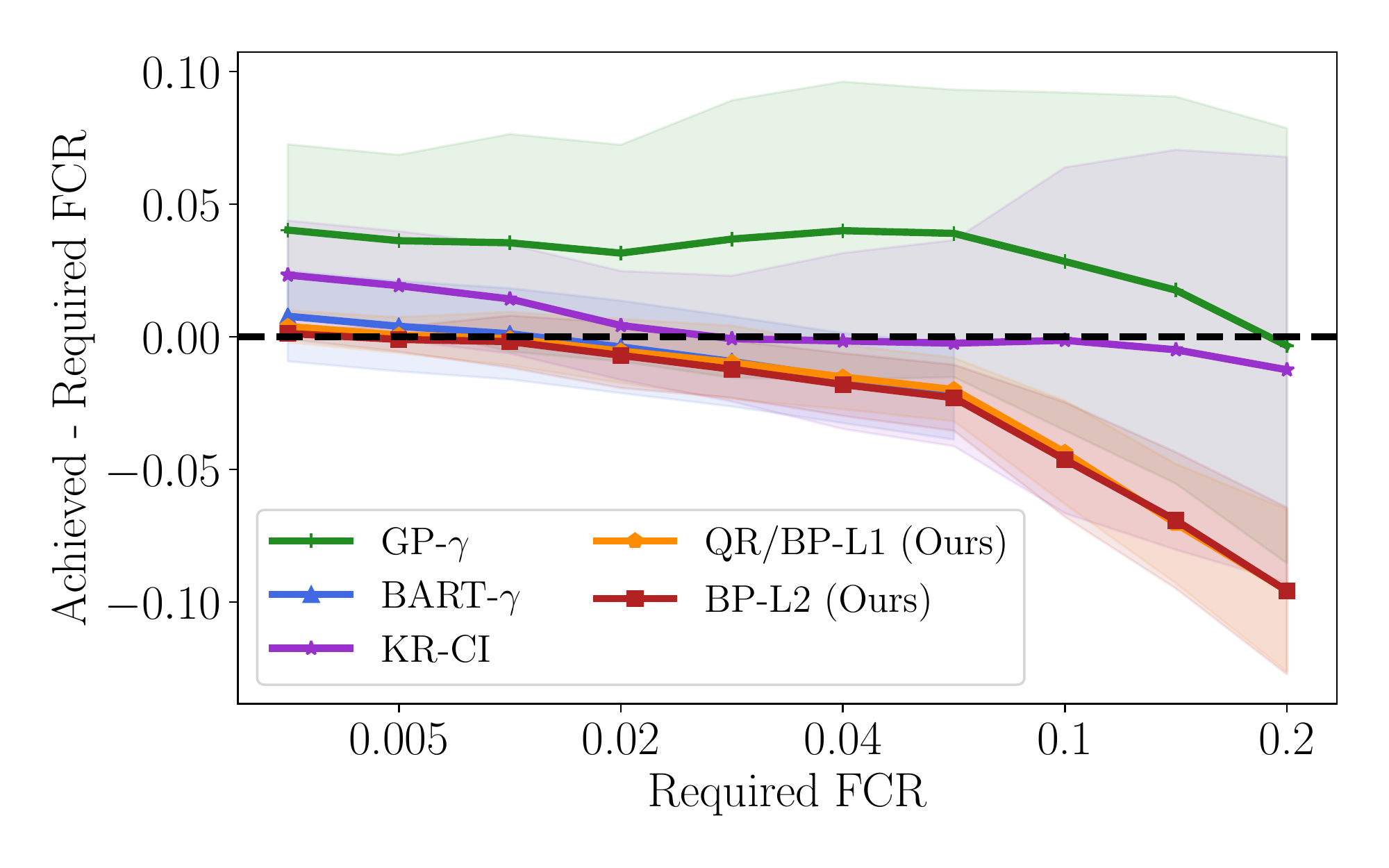}%

    \end{subfigure}
  \end{minipage}

  \caption{ACIC results. Plots show results averaged over 20 simulations. Plot~\ref{fig:acic0_1k} shows the mean interval width for different values of the achieved FCR on a held-out test set. Plot~\ref{fig:acic1_1k} shares the same legend as plot~\ref{fig:acic0_1k}, and shows the violation of the required FCR (= achieved - required) at different values of required FCR. Models above the dotted black line are in violation of the required FCR. The two plots show that BP achieves a mean interval width comparable to that of BART but at a lower violation of the required FCR. BP outperforms all kernel-based methods in terms of mean interval width and violation to the required FCR. CCI methods achieve the worst violations.
\vspace{-1.2em} \label{fig:acic_1k}}
\end{figure*}